\documentclass{colt2017} % Include author names

\usepackage{times}

\usepackage{algorithm}
\usepackage{algorithmic}
\usepackage{color}
\usepackage{graphicx}
\usepackage{epstopdf}
\usepackage[framemethod=tikz]{mdframed}
\usepackage{psfrag}

\newcommand{\I}{\ensuremath{\mathbf{I}}}

\newcommand{\w}{\ensuremath{\mathbf{w}}}
\newcommand{\x}{\ensuremath{\mathbf{x}}}
\newcommand{\y}{\ensuremath{\mathbf{y}}}
\newcommand{\z}{\ensuremath{\mathbf{z}}}
\newcommand{\0}{\ensuremath{\mathbf{0}}}

\newcommand{\bbE}{\ensuremath{\mathbb{E}}}

\newcommand{\bbR}{\ensuremath{\mathbb{R}}}

\newcommand{\calA}{\ensuremath{\mathcal{A}}}

\newcommand{\calN}{\ensuremath{\mathcal{N}}}
\newcommand{\calO}{\ensuremath{\mathcal{O}}}

\newcommand{\abs}[1]{\left\lvert#1\right\rvert}
\newcommand{\norm}[1]{\left\lVert#1\right\rVert}

{%
\begin{list}{#1}{
\vspace{-\topsep}
\vspace{-\partopsep}
\setlength{\itemindent}{0cm}
\setlength{\rightmargin}{0cm}
\setlength{\listparindent}{0cm}
\settowidth{\labelwidth}{#1}
\setlength{\leftmargin}{\labelwidth}
\addtolength{\leftmargin}{\labelsep}
\setlength{\itemsep}{0cm}
}%
}%
{%
\end{list}
\vspace{-\topsep}
\vspace{-\partopsep}
}

{\begin{enumerate}%
\renewcommand{\theenumi}{\roman{enumi}}%
\renewcommand{\labelenumi}{(\theenumi)}}%
{\end{enumerate}}

\newtheorem{thm}{Theorem}
\newtheorem{lem}[thm]{Lemma}
\newtheorem{prop}[thm]{Proposition}
\newtheorem{rmk}[thm]{Remark}
\DeclareMathOperator*{\argmin}{arg\,min}

\newcommand{\ie}{i.e.\@}
\newcommand{\eg}{e.g.\@}
\newcommand{\phit}{\phi_{I_t}}
\newcommand{\phii}{\phi_{I_t^{(i)}}}
\newcommand{\tw}{\tilde{\w}}
\newcommand{\hF}{\hat{F}}
\newcommand{\hG}{\hat{G}}
\newcommand{\hw}{\hat{\w}}
\newcommand{\bw}{\bar{\w}}

\usepackage{mathtools}
\DeclarePairedDelimiter{\ceil}{\lceil}{\rceil}

\usepackage{eqparbox}

\title[Efficient Distributed Stochastic Optimization]{Memory and Communication Efficient Distributed Stochastic Optimization with Minibatch-Prox}

 \coltauthor{
 \Name{Jialei Wang}\thanks{Equal contributions.} \Email{jialei@uchicago.edu}\\
 \addr University of Chicago
 \AND
 \Name{Weiran Wang}$^*$ \Email{weiranwang@ttic.edu}\\
 \addr Toyota Technological Institute at Chicago
 \AND
 \Name{Nathan Srebro} \Email{nati@ttic.edu}\\
 \addr Toyota Technological Institute at Chicago
 }

\begin{document}
\maketitle

\begin{abstract}
  We present and analyze an approach for distributed stochastic
  optimization which is statistically optimal and achieves near-linear
  speedups (up to logarithmic factors).  Our approach allows a
  communication-memory tradeoff, with either logarithmic communication
  but linear memory, or polynomial communication and a corresponding
  polynomial reduction in required memory.  This communication-memory
  tradeoff is achieved through minibatch-prox iterations (minibatch
  passive-aggressive updates), where a subproblem on a minibatch is
  solved at each iteration.  We provide a novel analysis for such a
  minibatch-prox procedure which achieves the statistical optimal
  rate regardless of minibatch size and smoothness, thus
  significantly improving on prior work.
\end{abstract}

\section{Introduction}
\label{sec:intro}

Consider the stochastic convex optimization (generalized learning)
problem~\citep{yudin1983problem,vapnik1995nature,Shalev_09a}:
\begin{align} \label{e:obj}
\min_{ \w \in \Omega }\; \phi (\w) := \bbE_{\xi \sim D} \left[ \ell (\w, \xi) \right]
\end{align}
where our goal is to learn a predictor $\w$ from the convex domain
$\Omega$ given the convex instantaneous (loss) function $\ell
(\w,\xi)$ and i.i.d. samples $\xi_1,\xi_2,\dots$ from some unknown
data distribution $D$.  When optimizing on a single machine,
stochastic approximation methods such as stochastic gradient descent
(SGD) or more generally stochastic mirror descent, are ideally suited
for the problem as they typically have optimal sample complexity
requirements, and run in linear time in the number of samples, and
thus also have optimal runtime.  Focusing on an $\ell_2$ bounded
domain with $B=\sup_{\w \in \Omega} \norm{\w}$ and $L$-Lipschitz loss,
the min-max optimal sample complexity is $n(\varepsilon) = \calO (L^2
B^2/\varepsilon^2)$, and this is achieved by SGD using
$\calO(n(\epsilon))$ vector operations.  Furthermore, if examples are
obtained one at a time (in a streaming setting or through access to a
``button'' generating examples), we only need to store $\calO(1)$ vectors
in memory.

The situation is more complex in the distributed setting where no
single method is known that is optimal with respect to sample
complexity, runtime, memory {\em and} communication.  Specifically,
consider $m$ machines where each machine $i=1,...,m$ receives samples
$\xi_{i1},\xi_{i2},...$ drawn from the same distribution $D$. This can
equivalently be thought of as randomly distributing samples across $m$
servers.  We also assume the objective is $\beta$-smooth, taking
$L,\beta=\calO(1)$ in our presentation of results.  The goal is to
find a predictor $\hat \w \in \Omega$ satisfying $\bbE \left[
  \phi(\hat \w) - \min_{\w \in \Omega} \phi(\w) \right] \le
\varepsilon$ using \emph{the smallest possible number of samples} per
machine, \emph{the minimal elapsed runtime}, and \emph{the smallest
  amount of communication}, and also \emph{minimal memory} on each
machine (again, when examples are received or generated one at a
time).  Ideally, we could hope for a method with linear speedup,
i.e.~$\calO(n(\epsilon)/m)$ runtime, using the statistically optimal number of samples
$\calO(n(\epsilon))$ and constant or near-constant communication and
memory.  Throughout we measure runtime in terms of
vector operations, memory in terms of number of vectors that need to
be stored on each machine and communication in terms of number of
vectors sent per machine\footnote{In all methods involved,
  communication is used to average vectors across machines and make
  the result known to one or all machines.  We are actually counting
  the number of such operations.}.  These resource requirements are
summarized in Table~\ref{table:summary}.

One simple approach for distributed stochastic optimization is
minibatch SGD~\citep{Cotter_11a,dekel2012optimal}, where in each
update we use a gradient estimate based on $mb$ examples: $b$ examples
from each of the $m$ machines.  Distributed minibatch SGD attains
optimal statistical performance with $\calO \left( n(\varepsilon)/m
\right)$ runtime, as long as the minibatch size is not too large:
\citet{dekel2012optimal} showed that the minibatch size can be as
large as $bm = \calO (\sqrt{n(\varepsilon)})$,
and~\citet{Cotter_11a} showed that with acceleration this
can be increased to $bm = \calO ( n(\varepsilon)^{3/4} )$.  Using this
maximal minibatch size for accelerated minibatch SGD thus yields a
statistically optimal method with linear speedup in runtime, $\calO
(1)$ memory usage, and $\calO (n(\varepsilon)^{1/4})$ rounds of
communication--see Table \ref{table:summary}.  This is
the most communication-efficient method with true linear speedup
we are aware of.

An alternative approach is to use distributed optimization to optimize the regularized empirical objective:
\begin{align}
  \min_{\w}\; \phi_{S}(\w) + \frac{\nu}{2} \norm{\w}^2, 
  \label{e:erm}
\end{align}
where $\phi_S$ is the empirical objective on $n(\epsilon)$ i.i.d.
samples, distributed across the machines and $\nu = \calO ( L/(B
\sqrt{n(\varepsilon)}))$.  A naive approach here is to use accelerate
gradient descent, distributing the gradient computations, but this, as
well as approaches based on ADMM~\citep{Boyd_11a}, are dominated by minibatch SGD
(\citealt{shamir2014distributed} and see also Table
\ref{table:summary}).  Better alternatives take advantage of the
stochastic nature of the problem: DANE~\citep{Shamir_14a} requires only
$\calO(B^2 m)$ rounds of communication for squared loss problems, while
DiSCO~\citep{zhang2015disco} and AIDE~\citep{reddi2016aide}) reduce this
further to $\calO(B^{1/2}m^{1/4})$ rounds of communication.  However,
these communication-efficient methods usually require expensive
computation on each local machine, solving an optimization problem on
all local data at each iteration.  Even if this can be done in
near-linear time, it is still difficult to obtain computational
speedup compared with single machine solution, and certainly not
linear speedups---see Table~\ref{table:summary}.  Furthermore, since
each round of these methods involves optimization over a fixed
training set, this training set must be stored thus requiring
$n(\varepsilon)/m$ memory per machine.

Designing stochastic distributed optimization problems with linear, or
near-linear, speedups, and low communication and memory requirements
is thus still an open problem.  We make progress in this paper
analyzing and presenting methods with near-linear speedups and better
communication and memory requirements.  As with the analysis of DANE,
DiSCO and AIDE, our analysis is rigorous only for least squared
problems, and so all results should be taken in that context (the
methods themselves are applicable to any distributed stochastic convex
optimization problem).

\begin{figure}[t]
  \centering
  \psfrag{Communication}[][]{Communication}
  \psfrag{Computation}[][]{Computation}
  \psfrag{Memory}[][]{Memory}
\psfrag{Stochastic DANE}[l][l]{MP-DANE}
\psfrag{Acc minibatch SGD}[l][l]{Acc. Mini. SGD}
\psfrag{DiSCO/AIDE}[l][l]{DiSCO/AIDE}
\psfrag{DSVRG}[l][l]{DSVRG}
\psfrag{MB-DSVRG}[l][l]{MP-DSVRG}
    \includegraphics[height=0.3\textheight, width=0.55\textwidth]{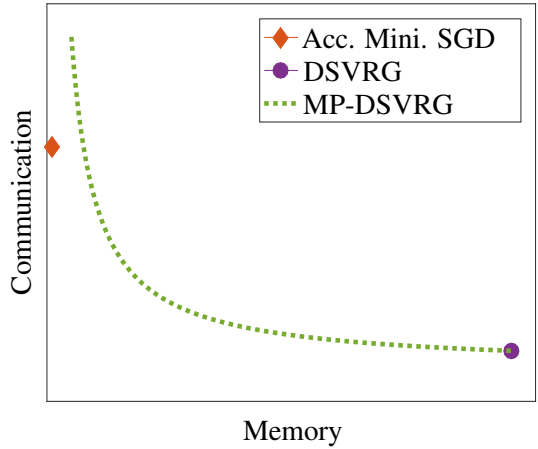}
        %  \makebox[1 \textwidth]{memory}
  \caption{Trade-offs  between memory and communication for the proposed MP-DSVRG approach.}
  \label{fig:cartoon_simple}
\end{figure}

\subsection*{Our contributions} 

\begin{itemize}
\item We first apply the recently proposed distributed SVRG (DSVRG) algorithm 
for regularized loss minimization to the distributed stochastic convex
optimization problem, and show that on least square problems it can achieve near-linear speedup
with very low communication, but with high memory cost---see DSVRG
in Table \ref{table:summary}.
\item We propose a novel algorithm that improves the
  memory cost, which we call minibatch-prox with DSVRG (MP-DSVRG). 
  For least square problems it
  achieves near-linear speedup with communication cost that is higher
  than DSVRG but increases only logarithmically with $n(\varepsilon)$,
  but with much lower memory requirements.  Moreover, our algorithm is
  flexible, allowing to trade off between communication and memory (depicted in Figure \ref{fig:cartoon_simple}),
  without affecting the computational efficiency.  Our method is based
  on careful combinations of inexact minibatch proximal update,
  communication-efficient optimization and linearly convergent
  stochastic gradient algorithms for finite-sums.
\item As indicated above, our method is based on minibatch proximal
  update. That is, a minibatch approach where in each iteration a
  non-linearized problem is solved on a stochastic minibatch.  This
  can be viewed as a minibatch generalization to the
  passive-aggressive algorithm~\citep{crammer2006online} and has been considered in
  various contexts~\citep{kulis2010implicit,toulis2014implicit}.  We show that such an
  approach achieves the 
  optimal statistical rate in terms of the number of samples used
  independent of the number of iterations, i.e.~with \emph{any}
  minibatch size.  This significantly improves over the previous
  analysis of~\citet{Li_14e}, as the guarantee is better, it entirely
  avoid the dependence on the minibatch size and does not rely on
  additional assumptions as in~\citet{Li_14e}.  The guarantee holds
  for any Lipschitz (even non-smooth) objective.  Furthermore, to make
  the minibatch proximal iterate more practical and useful in
  distributed setting, we also extend the analysis to algorithms which
  solve each minibatch subproblem inexactly.  Our analysis of exact
  and inexact minibatch proximal updates may be of
  independent interest and useful in other contexts and as a basis for
  other methods.
\end{itemize}

\if 0
\begin{table*}[t]
  \centering
  \begin{tabular}{|c||c|c|c|c|}
    \hline
    & Samples & Communication & Computation & Memory  \\\hline\hline
    Ideal Solution & $n(\varepsilon)$  & $\calO (1)$ & $n(\varepsilon)/m$  & $\calO (1)$ \\\hline
    Accelerated GD  &  $n(\varepsilon)$   &  $B^{1/2} n(\varepsilon)^{1/4}$  &   $B^{1/2}n(\varepsilon)^{5/4}/m $ & $n(\varepsilon)/m$ \\
    % \hline
    Acc.~Minibatch SGD
    &  $n(\varepsilon)$    &  $B^{1/2} n(\varepsilon)^{1/4}$   &   $n(\varepsilon)/m$ & $\calO (1)$ \\
    DANE  &  $n(\varepsilon)$  & $B^2 m$  &   $B^2 n(\varepsilon)$ &$n(\varepsilon)/m$ \\ 
    DiSCO  &  $n(\varepsilon)$  & $B^{1/2} m^{1/4}$  &   $B^{1/2} n(\varepsilon)/m^{3/4}$ &$n(\varepsilon)/m$ \\ 
    AIDE &  $n(\varepsilon)$  & $B^{1/2} m^{1/4}$  &   $B^{1/2} n(\varepsilon)/m^{3/4}$ &$n(\varepsilon)/m$ \\ 
    \hline
    DSVRG &  $n(\varepsilon)$  & $\calO (1)$  & $n(\varepsilon)/m$ &$n(\varepsilon)/m$ \\ \hline
    MP-DANE ($b\leq b_{\operatorname{mp-dane}}$) &  $n(\varepsilon)$  & $n(\varepsilon)/(mb)$  &   $n(\varepsilon)/m$ &$b$ \\
    MP-DANE ($b=b_{\operatorname{mp-dane}}$) & $n(\varepsilon)$ & $B^2 m$ & $n(\varepsilon)/m$  &$n(\varepsilon)/(m^2 B^2)$ \\
    \hline
  \end{tabular}
  \caption{Summary of resources required by different approaches to
    distributed stochastic least squares problems, in units of vector 
    operations/communications/memory per machine, ignoring constants and log-factors, here $b_{\operatorname{mp-dane}} \asymp n(\varepsilon)/(m^2 B^2)$.}
  \label{table:summary}
\end{table*}
\fi

\begin{table*}[t]
  \centering
  \begin{tabular}{|c||c|c|c|c|}
    \hline
    & Samples & Communication & Computation & Memory  \\\hline\hline
    Ideal Solution & $n(\varepsilon)$  & $\calO (1)$ & $n(\varepsilon)/m$  & $\calO (1)$ \\\hline
    Accelerated GD  &  $n(\varepsilon)$   &  $B^{1/2} n(\varepsilon)^{1/4}$  &   $B^{1/2}n(\varepsilon)^{5/4}/m $ & $n(\varepsilon)/m$ \\
    % \hline
    Acc.~Minibatch SGD
    &  $n(\varepsilon)$    &  $B^{1/2} n(\varepsilon)^{1/4}$   &   $n(\varepsilon)/m$ & $\calO (1)$ \\
    DANE  &  $n(\varepsilon)$  & $B^2 m$  &   $B^2 n(\varepsilon)$ &$n(\varepsilon)/m$ \\ 
    DiSCO  &  $n(\varepsilon)$  & $B^{1/2} m^{1/4}$  &   $B^{1/2} n(\varepsilon)/m^{3/4}$ &$n(\varepsilon)/m$ \\ 
    AIDE &  $n(\varepsilon)$  & $B^{1/2} m^{1/4}$  &   $B^{1/2} n(\varepsilon)/m^{3/4}$ &$n(\varepsilon)/m$ \\ 
    \hline
    DSVRG &  $n(\varepsilon)$  & $\calO (1)$  & $n(\varepsilon)/m$ &$n(\varepsilon)/m$ \\ \hline
    MP-DSVRG ($b\leq b_{\max}$) &  $n(\varepsilon)$  & $n(\varepsilon)/(mb)$  &   $n(\varepsilon)/m$ &$b$ \\
    MP-DSVRG ($b=b_{\max}$) & $n(\varepsilon)$ & $\calO (1)$ & $n(\varepsilon)/m$  &$n(\varepsilon)/m$ \\
    \hline
  \end{tabular}
  \caption{Summary of resources required by different approaches to
    distributed stochastic least squares problems, in units of vector 
    operations/communications/memory per machine, ignoring constants and log-factors, here $b_{\max} = n(\varepsilon)/m $.}
  \label{table:summary}
\end{table*}

\paragraph{Notations} We denote by $\w_* = \argmin_{ \w \in \Omega }\ \phi (\w)$ the optimal solution to~\eqref{e:obj}. Throughout the paper, we assume the instantaneous function $\ell (\w, \xi)$ is $L$-Lipschitz and $\lambda$-strongly convex in $\w$ for some $\lambda \ge 0$ on the domain $\Omega$: 
\begin{gather*}
\abs{ \ell(\w, \xi) - \ell(\w^\prime, \xi) } \le L \norm{\w - \w^\prime}, \\
\ell(\w, \xi) - \ell(\w^\prime, \xi) \ge \left< \nabla \ell (\w^\prime, \xi),\, \w - \w^\prime \right> + \frac{\lambda}{2} \norm {\w - \w^\prime}^2,\qquad \forall \w, \w^\prime \in \Omega.
\end{gather*}
Sometimes we also assume $\ell(\w,\xi)$ is $\beta$-smooth in $\w$:
 \begin{align*} % \label{eqn:smoothness}
 \ell(\w, \xi) - \ell(\w^\prime, \xi) \leq \left< \nabla \ell (\w^\prime, \xi),\, \w - \w^\prime \right> + \frac{\beta}{2} \norm {\w - \w^\prime}^2, \qquad \forall \w,\w^\prime \in \Omega.
 \end{align*}
For distributed stochastic optimization, our analysis focuses on the least squares loss $\ell(\w,\xi) = \frac{1}{2} (\w^\top \x - y)^2$ where $\xi = (\x, y)$. 
\section{Distributed SVRG for stochastic convex optimization}
\label{sec:dsvrg}

Recently, \citet{lee2015distributed} suggested using fast
randomized optimization algorithms for finite-sums, and in particular
the SVRG algorithm, as a distributed optimization approach for ~\eqref{e:erm}. The
authors noted that, for SVRG, when the the sample size $n
(\varepsilon)$ dominates the problem's condition number $\beta/\nu$
where $\beta$ is the smoothness parameter of $\ell(\w,\xi)$,
the time complexity is dominated by computing the batch gradients.
This operation can be trivially parallelized.  The stochastic updates,
on the other hand, can be implemented on a single machine while the other
machines wait, with the only caveat being that only
sampling-without-replacement can be implemented this way.  The use of
without-replacement sampling was theoretically justified in a recent
analysis by~\citet{shamir2016without}.

In the distributed stochastic convex optimization setting considered here, 
DSVRG in fact achieves linear speedup in certain regime as follows.
In each iteration of the algorithm, each machine first computes its
local gradient and average them with one communication round to obtain
the global batch gradient, and then a \emph{single} machine performs
the SVRG stochastic updates by processing its local data once
(sampling the $n (\varepsilon) / m$ examples without replacement). By
the linear convergence of SVRG, as long as the number of stochastic
updates $n (\varepsilon) / m$  is larger than $\beta/\nu = \calO (
\beta B \sqrt{n (\varepsilon)} / L )$, the algorithm converges to
$\calO (\epsilon)$-suboptimality (in both the empirical and stochastic
objective) in $\calO (\log 1 / \varepsilon) = \calO \left( \log n
(\varepsilon) \right)$ iterations; and this condition is satisfied\footnote{If $n(\varepsilon) \gtrsim m^2$ does not hold, we can use a ``hot-potato'' style algorithm where we process all data once on machine $i$ and pass the predictor to machine $i+1$ until we obtain sufficiently many stochastic updates. But then the computation efficiency deteriorates and we no longer have linear speedup in runtime.} for $n(\varepsilon) \gtrsim m^2$.

Clearly, in the above regime, each iteration of DSVRG uses two rounds
of communications and the total communication complexity is $\calO
\left( n (\varepsilon) \right)$. On the other hand, the computation
for each machine is compute the local gradient (in time  $\calO (n
(\varepsilon) / m)$) in each iteration, resulting in a total time
complexity of
  $\calO (n (\varepsilon) \log n (\varepsilon) / m)$. This explains
the DSVRG entry in Table~\ref{table:summary}.

Being communication- and computation-efficient, DSVRG requires each
machine to store a portion of the sample set for ERM to make multiple
passes over them, and is therefore not memory-efficient. In fact, this
disadvantage is shared by previously known communication-efficient
distributed optimization algorithms, including DANE, DiSCO, and AIDE.
In order to develop a memory- and communication-efficient algorithm
for distributed stochastic optimization, we need to bypass the ERM
setting and this is enabled by the following minibatch-prox algorithm.

%\weiran{Need a smooth transition to the next section.}
\section{The minibatch-prox algorithm for stochastic optimization}

\label{sec:minibatch-prox}

In this section, we describe and analyze the minibatch-prox algorithm for stochastic optimization, which allows us to use arbitrarily large minibatch size without slowing down the convergence rate. We first present the basic version where each proximal objective is solved exactly for each minibatch, which achieves the optimal convergence rate. Then, we show that if each minibatch objective is solved accurately enough, the algorithm still converges at the optimal rate, opening the opportunity for efficient implementations. 

\subsection{Exact minibatch-prox}
\label{sec:exact}

The ``exact'' minibatch-prox is defined by the following iterates: for $t=1,\dots,$
\begin{gather} 
  \w_t = \argmin_{\w \in \Omega} \; f_t (\w), \nonumber \\ \label{e:exact-update}
  \text{where}\quad f_t(\w) := \phit(\w) + \frac{\gamma_t}{2} \norm{\w - \w_{t-1}}^2 =  \frac{1}{b} \sum_{\xi \in I_t} \ell (\w, \xi) +  \frac{\gamma_t}{2} \norm{\w - \w_{t-1}}^2 ,
\end{gather} 
$\gamma_t>0$ is the (inverse) stepsize parameter at time $t$, and $I_t$ is a set of a $b$ samples from the unknown distribution $D$. 
To understand the updates in~\eqref{e:exact-update}, we first observe by the first order optimality condition for $f_t (\w)$ that
\begin{align} \label{e:exact-condition-t}
  \nabla \phit (\w_t) + \gamma_t (\w_t - \w_{t-1}) \in - \calN_{\Omega} ( \w_t ),
\end{align} 
where $\nabla \phit (\w_t)$ is some subgradient of $\phit (\w)$ at $\w_t$, and $\calN_{\Omega} ( \w_t ) = \left\{ \y | \left< \w - \w_t,\, \y \right> \le 0,\, \forall \w \in \Omega \right\}$ is the normal cone of $\Omega$ at $\w_t$. Equivalently, the above condition implies 
\begin{align} \label{e:exact-equivalent-sgd}
  \w_t = P_{\Omega} \left( \w_{t-1} - \frac{1}{\gamma_t} \nabla \phit (\w_t) \right),
\end{align}
where  $P_{\Omega} (\w)$ denotes the projection of $\w$ onto $\Omega$. The update rule~\eqref{e:exact-equivalent-sgd} resembles that of the standard minibatch gradient descent, except the gradient is evaluated at the ``future'' iterate. 

Proximal steps, of the form \eqref{e:exact-update} or equivalently
\eqref{e:exact-equivalent-sgd}, are trickier to implement compared to
(stochastic) gradient steps, as they involve optimization of a
subproblem, instead of merely computing and adding gradients.
Nevertheless, they have been suggested, used and studied in several
contexts.  \citet{crammer2006online} proposed the ``passive
aggressive'' update rule, where a margin-based loss from a single
example with a quadratic penalty is minimized---this corresponds to
\eqref{e:exact-update} with a ``batch size'' of one.  More general
loss functions, still for ``batch sizes'' of one, were also analyzed
in the online learning
setting~\citep{cheng2006implicit,kulis2010implicit}. % \citet{toulis2014implicit} analyzed the statistical efficiency and variance for implicit SGD.
For finite-sum objectives, methods based on incremental/stochastic
proximal updates were studied by~\citet{bertsekas2011incremental,Bertsek15a,
  defazio2016simple}. 
\citet{needell2014paved} analyzed a randomized block Kaczmarz method
in the context of solving linear systems, which also minimizes the
empirical loss on a randomly sampled minibatch. To the best of our
knowledge, no prior work has analyzed the general minibatch variant of
proximal updates for stochastic optimization except~\citet{Li_14e}.
However, the analysis of~\citet{Li_14e} assumes a stringent condition
which is hard to verify (and is often violated) in practice, which we
will discuss in detail in this section.

The following lemma provides the basic property of the update at each iteration. 
\begin{lem} \label{lem:exact-distance-to-any}
  For any $\w \in \Omega$, we have
  \begin{align} \label{e:exact-distance-to-any}
    \frac{\lambda + \gamma_t}{\gamma_t} \norm{ \w_t - \w }^2 
    \le \norm{\w_{t-1} - \w}^2 - \norm{ \w_{t-1} - \w_t }^2 - \frac{2}{\gamma_t} \left( \phit (\w_t) - \phit (\w) \right).
  \end{align}
\end{lem}

To derive the convergence guarantee, we need to relate $\phit (\w_t)$ to $\phi (\w)$. 
The analysis of~\citet{Li_14e} for minibatch-prox made the assumption that for all $t \ge 1$:
\begin{align}
  \bbE_{I_t} \left[ D_{\phi} (\w_t; \w_{t-1}) \right] \le \bbE_{I_t} \left[ D_{\phit} (\w_t; \w_{t-1}) \right] + \frac{\gamma_t}{2} \norm{\w_t - \w_{t-1}}^2,
  \label{eqn:li_condition}
\end{align}
where $D_{f} (\w, \w^\prime) = f(\w) - f(\w^\prime) - \left< \nabla f(\w^\prime),\, \w - \w^\prime \right>$ denotes the Bregman divergence defined by the potential function $f$. This condition is hard to verify, and may constrain the stepsize to be very small. For example, as the authors argued, if $\ell(\w,\xi)$ is $\beta$-smooth with respect to $\w$, we have
\begin{align*}
  D_{\phi}(\w_t;\w_{t-1}) \leq \frac{\beta}{ 2 } \norm{\w_t - \w_{t-1}}^2,
\end{align*}
and combined with the fact that $D_{\phit} (\w_t; \w_{t-1}) \ge 0$, one can guarantee~\eqref{eqn:li_condition} by setting $\gamma_t \geq \beta$. However, to obtain the optimal convergence rate, \citet{Li_14e} needed to set $\gamma_t = \calO (\sqrt{T/b})$ which would imply $b = \calO (T)$ in order to have $\gamma_t \ge \beta$. In view of this implicit constraint that the minibatch size $b$ can not be too large, the analysis of~\citet{Li_14e} does not really show advantage of minibatch-prox over minibatch SGD, whose optimal minibatch size is precisely $b=\calO(T)$.

Our analysis is free of any additional assumptions. The key observation is that, when $b$ is large, we expect $\phit (\w)$ to be close to $\phi (\w)$. Define the stochastic objective 
\begin{align} \label{e:exact-update-expt}
  F_t (\w) := \bbE_{I_t} \left[ f_t (\w) \right] = \phi (\w) + \frac{\gamma_t}{2} \norm{\w - \w_{t-1}}^2.
\end{align}
Then $\w_t$ is the ``empirical risk minimizer'' of $F_t (\w)$ as it solves the empirical version $f_t (\w)$ with $b$ samples. Using a stability argument~\citep{Shalev_09a}, we can establish the ``generalization'' performance for the (inexact) minimizer of the minibatch objective.

\begin{lem} \label{lem:stability-at-time-t}
  For the minibatch-prox algorithm,we have
  \begin{align*}
    \abs{ \bbE_{I_t} \left[ \phi (\w_t) - \phit (\w_t) \right] } \le \frac{4 L^2}{(\lambda + \gamma_t) b}.
  \end{align*}
  Moreover, if a possibly randomized algorithm $\calA$ minimizes $f_t (\w)$ up to an error of $\eta_t$, \ie, $\calA$ returns an approximate solution $\tw_t$ such that $ \bbE_\calA \left[ f_t (\tw_t) -  f_t (\w_t)\right] \le \eta_t $, we have 
  \begin{align*}
    \abs{ \bbE_{I_t,\calA} \left[ \phi (\tw_t) - \phit (\w_t) \right] } \le \frac{4 L^2}{(\lambda + \gamma_t) b} + \sqrt{\frac{2 L^2 \eta_t}{\lambda + \gamma_t}}.
  \end{align*}
\end{lem}
Combining Lemma~\ref{lem:exact-distance-to-any} and Lemma~\ref{lem:stability-at-time-t}, we obtain the following key lemma regarding the progress on the stochastic objective at each iteration of minibatch-prox.

\begin{lem} \label{lem:exact-distance-reduction-with-stability}
  For iteration $t$ of exact minibatch-prox, we have for any $\w \in \Omega$ that 
  \begin{align} \label{e:exact-distance-reduction-with-stability}
    \frac{\lambda + \gamma_t}{\gamma_t}  \bbE_{I_t} \norm{ \w_t - \w }^2 \le 
    \norm{\w_{t-1} - \w}^2 - \frac{2}{\gamma_t} \bbE_{I_t} \left[ \phi (\w_t) - \phi (\w) \right] + \frac{8 L^2}{\gamma_t (\lambda + \gamma_t) b}.
  \end{align}
\end{lem}
We are now ready to bound the overall convergence rates of minibatch-prox.

\begin{thm}[Convergence of exact minibatch-prox --- weakly convex $\ell (\w,\xi)$] 
  \label{thm:rate-exact-minibatch-weakly-convex}
  For $L$-Lipschitz instantaneous function $\ell (\w,\xi)$, set $\gamma = \sqrt{\frac{8T}{b}} \cdot \frac{L}{\norm{\w_0 - \w_*}}$ for $t=1,\dots,T$ in minibatch-prox. Then for $\widehat{\w}_T = \frac{1}{T} \sum_{t=1}^T \w_t$, we have
  \begin{align*}
    \bbE \left[ \phi(\widehat{\w}_T) - \phi (\w_*) \right] \le \frac{\sqrt{8} L}{\sqrt{bT}} \norm{\w_0 - \w_*}.
  \end{align*}
\end{thm}

\begin{thm}[Convergence  of exact minibatch-prox --- strongly convex $\ell (\w,\xi)$] 
  \label{thm:rate-exact-minibatch-strongly-convex}
  For $L$-Lipschitz and $\lambda$-strongly convex instantaneous function $\ell (\w,\xi)$, set $\gamma_t = \frac{\lambda (t-1)}{2}$ for $t=1,\dots,T$ in minibatch-prox. Then for $\widehat{\w}_T = \frac{2}{T (T+1)} \sum_{t=1}^T t \w_t$, we have
  \begin{align*}
    \bbE \left[ \phi(\widehat{\w}_T) - \phi (\w_*) \right] 
    \le \frac{16 L^2 }{\lambda b (T+1)}.
  \end{align*}
\end{thm}

\subsection{Inexact minibatch-prox}
\label{sec:inexact}

We now study the case where instead of solving the subproblems $f_t (\w)$ exactly, we only solve it approximately to sufficient accuracy. 
The ``inexact'' minibatch-prox uses a possibly randomized algorithm $\calA$ for approximately solving one subproblem on a minibatch in each iteration, and generates the following iterates: for $t=1,\dots,$
\begin{gather} \label{e:inexact-update}
  \tw_t \approx \bw_t :=\argmin_{\w \in \Omega} \; \tilde{f}_t (\w) \qquad \text{where}\quad \tilde{f}_t (\w) := \phit(\w) + \frac{\gamma_t}{2} \norm{\w - \tw_{t-1}}^2, \\
  \text{and}\qquad  \bbE_{\calA} \left[ \tilde{f}_t (\tw_t) -  \tilde{f}_t (\bw_t) \right] \le \eta_t.  \nonumber
\end{gather} 
Analogous to Lemma~\ref{lem:exact-distance-reduction-with-stability}, we can derive the following lemma using stability of inexact minimizers.
\begin{lem} \label{lem:inexact-distance-to-any}
  Fix any $\w \in \Omega$. For iteration $t$ of inexact minibatch-prox, we have 
  \begin{align} 
    \bbE_{I_t,\calA} \left[ \phi (\tw_t) - \phi (\w) \right] 
    & \le \frac{\gamma_t}{2} \bbE_{I_t,\calA} \norm{\tw_{t-1} - \w}^2 
    - \frac{\lambda + \gamma_t}{2} \bbE_{I_t,\calA} \norm{ \tw_t - \w }^2 
    + \frac{4 L^2}{ (\lambda + \gamma_t) b}  \nonumber \\ \label{e:inexact-distance-to-any}
    &\qquad\qquad + \sqrt{\frac{2 L^2 \eta_t}{\lambda + \gamma_t}} + \sqrt{2 (\lambda + \gamma_t) \eta_t } \cdot \sqrt{ \bbE_{I_t,\calA} \norm{ \tw_t - \w }^2 } .
  \end{align}
\end{lem}
Note that when $\eta_t=0$, the above guarantee reduces to that of exact minibatch-prox.

We now show that when the minibatch subproblems are solved sufficiently accurately, we still obtain the $\calO ( 1/\sqrt{bT} )$ rate for weakly-convex loss and $\calO ( 1/(\lambda bT) )$ rate for strongly-convex loss.

\begin{thm}[Convergence of inexact minibatch-prox --- weakly convex $\ell (\w,\xi)$]
  \label{thm:inexact-weakly-convex-sufficient}
  For $L$-Lipschitz instantaneous function $\ell (\w,\xi)$, set $\gamma_t = \gamma = \sqrt{\frac{8T}{b}} \cdot \frac{L}{\norm{\w_0 - \w_*}}$ for all $t \ge 1$ in inexact minibatch-prox. Assume that for all $t\ge 1$, the error in minimizing $\tilde{f}_t (\w)$ satisfies for some $\delta>0$ that 
  \begin{align*}
    \bbE_{\calA} \left[ \tilde{f}_t (\tw_t) - \min_{\w} \tilde{f}_t (\w) \right]  % \eta_t 
    \le \min \left( c_1 \left( \frac{T}{b} \right)^{\frac{1}{2}}, \, c_2 \left( \frac{T}{b} \right)^{\frac{3}{2}}  \right) 
    \cdot \frac{ L \norm{\tw_0 - \w_*}}{t^{2+2\delta}}. 
  \end{align*} 
  Then for $\widehat{\w}_T = \frac{1}{T} \sum_{t=1}^T \tw_t$, we have 
  $
  \bbE \left[ \phi(\widehat{\w}_T) - \phi (\w_*) \right] \le \frac{c_3  L \norm{\w_0 - \w_*}}{\sqrt{bT}},
  $ 
  where $c_3$ only depends on $c_1,c_2$ and $\delta$. 
  For example, by setting $c_1 = 10^{-4}, c_2 = 10^{-4}, \delta = 1/2$, we have
  \[
  \bbE \left[ \phi(\widehat{\w}_T) - \phi (\w_*) \right] \le \frac{ \sqrt{10}  L \norm{\w_0 - \w_*}}{\sqrt{bT}}.
  \]
\end{thm}

\begin{thm}[Convergence of inexact minibatch-prox --- strongly convex $\ell (\w,\xi)$]
  \label{thm:inexact-strongly-convex-sufficient}
  For $L$-Lipschitz and $\lambda$-strongly convex instantaneous function $\ell (\w,\xi)$, set $\gamma_t = \frac{\lambda (t-1)}{2}$ for $t=1,\dots$ in inexact minibatch-prox. Assume that for all $t\ge 1$, the error in minimizing $\tilde{f}_t (\w)$ satisfies for some $\delta>0$ that 
  \begin{align*}
    \bbE_{\calA} \left[ \tilde{f}_t (\tw_t) - \min_{\w} \tilde{f}_t (\w) \right] %  \eta_t 
    \le \min \left( c_1 \left( \frac{T}{b} \right),\, c_2 \left( \frac{T}{b} \right)^2 \right) \cdot  \frac{L^2 }{t^{3+2\delta} \lambda }. 
  \end{align*}
  Then for $\widehat{\w}_T = \frac{2}{T (T+1)} \sum_{t=1}^T t \tw_t$, we have 
  $
  \bbE \left[ \phi(\widehat{\w}_T) - \phi (\w_*) \right] \le \frac{c_3 L^2}{\lambda b T},
  $ 
  where $c_3$ only depends on $c_1,c_2$ and $\delta$. 
  % Moreover, if we choose $\delta = 1/2$, $c_1 = c_2 = 10^{-4}$, we have
  % \[
  % \bbE \left[ \phi(\widehat{\w}_T) - \phi (\w_*) \right] \le \frac{16.5 L^2}{\lambda b T}.
  % \]
\end{thm}

\begin{rmk}
  The final inequalities in
  Theorem~\ref{thm:rate-exact-minibatch-weakly-convex}
  and~\ref{thm:inexact-weakly-convex-sufficient} actually apply more
  generally to all predictors in the domain.  That is, our proofs still hold
  with $\w^*$ replaced by any $\w \in \Omega$:
  \begin{align*}
    \bbE \left[ \phi(\widehat{\w}_T) - \phi (\w) \right] \le \calO \left( \frac{ L \norm{\w_0 - \w}}{\sqrt{bT}} \right), \qquad \w \in \Omega.
  \end{align*}
  This allows us to compete with any predictor in the domain (other than the minimizer). 
  For example, in order to compete on $\phi(\w)$ with the set of predictors with small norm $\left\{ \w: \norm{\w} \le B \right\}$, we can set the domain $\Omega=\bbR^d$ and initialize with $\w_0=\0$. In view of the above inequality, we still obtain the optimal rate $\calO \left( \frac{LB}{\sqrt{b T}} \right)$ from minibatch-prox by solving simpler, unconstrained subproblems (though we might have $\norm{\hw_T} > B$).
\end{rmk}

\section{Communication-efficient distributed minibatch-prox with SVRG}
\label{sec:mbdsvrg}

\begin{algorithm}[!t]
  \caption{Minibatch-prox with DSVRG for distributed stochastic convex optimization.}
  \label{alg:mbdsvrg}
  \renewcommand{\algorithmicrequire}{\textbf{Input:}}
  \renewcommand{\algorithmicensure}{\textbf{Output:}}
  \begin{algorithmic}
    \STATE Initialize $\w_0=\0$.
    \FOR{$t=1,2,\dots,T$}
    \STATE \% Outer loop performs minibatch-prox.
    \STATE Each machine $i$ draws a minibatch $I_t^{(i)}$ of $b$ samples from the underlying data distribution, and split $I_t^{(i)}$ to $p_i$ batches of size $b/p_i$: $B^{(i)}_{1},B^{(i)}_{2},...,B^{(i)}_{p_i}$
    \STATE Initialize $\z_0 \leftarrow \w_{t-1}, \quad \x_0 \leftarrow \w_{t-1}, \quad j\leftarrow 1, \quad s\leftarrow 1$
    \FOR{$k=1,2,\dots,K$}
    \STATE 1. All machines perform one round of communication to compute the average gradient: 
    \begin{align*} \nabla \phit (\z_{k-1}) \leftarrow \frac{1}{m} \sum_{i=1}^m \nabla \phii (\z_{k-1})
    \end{align*}
    \vspace{-2ex}
    \STATE 2. Machine $j$ performs stochastic updates by going through $B^{(j)}_{s}$ once without replacement:
    %% \FOR{$l \in B^{(j)}_{s}$}
    %% \STATE
    \[ \x_r \leftarrow \x_{r-1} - \eta \left( \nabla \ell (\x_{r-1}, \xi_l) - \nabla \ell (\z_{k-1}, \xi_l) + \nabla \phit (\z_{k-1}) +  \gamma ( \x_{r-1} - \w_{t-1} ) \right)
    \]
    for $\xi_l \in B^{(j)}_{s}$.
    %% \vspace{-2ex}
    %% \ENDFOR
    \STATE 3. Machine $j$ update $\z_{k}$:
    \[
    \z_{k} \leftarrow \frac{1}{|B^{(j)}_{s}|} \sum_{r=0}^{|B^{(j)}_{s}|} \x_{r},
    \]
    and broadcast $\z_{k}$ to other machines.
    \STATE 4. Update indices: $s \leftarrow s+1$,
    \IF{$s > p_j$}
    \STATE $s \leftarrow 1,\quad j \leftarrow j + 1.$
    \ENDIF
    \ENDFOR
    \STATE Update $\w_t \leftarrow \z_K$. 
    \ENDFOR
    \ENSURE $\w_T$ is the approximate solution.
  \end{algorithmic}
\end{algorithm}

We now apply the theoretical results of minibatch-prox to the distributed stochastic learning setting, and propose a novel algorithm that is both communication and computation efficient, and being able to explore trade-offs between memory and communication efficiency.

Suppose we have $m$ machines in a distributed environment. For each outer loop of our algorithm, each machine $i$ draws a minibatch $I_t^{(i)}$ of $b$ samples independently from other machines, and denote $I_t = \cup_{i=1}^m I_t^{(i)}$ which contains $bm$ samples. To apply the minibatch-prox algorithm from the previous section, we need to find an approximate solution to the following problem:
\begin{align}  \label{e:bm-minibatch-problem} 
  \min_{\w}\; \tilde f_t(\w) :=
  \phi_{I_t} (\w) + \frac{\gamma}{2} \norm{ \w - \w_{t-1}}^2.
\end{align}

Since the objective~\eqref{e:bm-minibatch-problem} involves functions from different machines, we use distributed optimization algorithms for solving it. In~\citet{Li_14e}, the authors proposed a simple algorithm EMSO to approximately solve~\eqref{e:bm-minibatch-problem}, where each machine first solve its own local objective, \ie, 
\begin{align}
  \label{e:distributed-emso-local-objective}
  \w_t^{(i)} = \argmin_{\w}\; \phi_{I_t^{(i)}} + \frac{\gamma}{2} \norm{ \w - \w_{t-1}}^2,
\end{align}
and then all machines average their local solutions via one round of communication: $\w_t = \frac{1}{m} \sum_{i=1}^m \w_t^{(i)}.$

We note that this can be considered as the ``one-shot-averaging'' approach~\citep{zhang2012communication} for solving~\eqref{e:bm-minibatch-problem}. Although this approach was shown to work well empirically, no convergence guarantee for the original stochastic objective~\eqref{e:obj} was provided by~\citet{Li_14e}. Here we instead use the distributed SVRG (DSVRG) algorithm~\citep{lee2015distributed,shamir2016without} to approximately solve~\eqref{e:bm-minibatch-problem}, as DSVRG enjoys excellent communication and computation cost when the problem is well conditioned (cf. Table~\ref{table:summary}).\footnote{It is also possible to equip minibatch-prox with other communication-efficient distributed optimization algorithms, for example in Appendix~\ref{sec:stochastic-dane}, we present a minibatch-prox DANE (MP-DANE) algorithm which uses the accelerated DANE method for solving~\eqref{e:bm-minibatch-problem}.}

We detail our algorithm, named MP-DSVRG (minibatch-prox with DSVRG), in Algorithm~\ref{alg:mbdsvrg}. The algorithm consists of two nested loops, where $t$, $k$ are iteration counters for minibatch-prox (the outer \textbf{for}-loop), and DSVRG (the inner \textbf{for}-loop) respectively. In each outer loop, each machine draws a minibatch $I_t^{(i)}$ to form the objective~\eqref{e:bm-minibatch-problem}, which will be solved approximately by the inner loops. Moreover, each machine splits its local dataset into $p_i$ batches: $I^{(i)}=\cup_{j=1}^{p_i} B^{(i)}_j$. In each inner loop, all machines communicate to calculate the global gradient (averaged local gradients) of~\eqref{e:bm-minibatch-problem}, and then one of the machines $j$ picks a local batch $B^{(j)}_{s}$ to perform the stochastic updates, where the local batch contains enough samples such that one pass of stochastic updates on $B^{(j)}_{s}$ decrease the objective quickly. We perform two rounds of communication in each inner loop, one for computing the global gradient, and one for broadcasting the new predictor obtained by a machine $j$.
%% and the amount of data we communicate per round is proportional to the number of features of the problem. 
As we will show in the next section, by carefully choosing the parameters, %% $\gamma,b,p_i$,  
we will obtain a convergent algorithm for distributed stochastic convex optimization with better efficiency guarantees than previous methods.

We now present detailed analysis for the computation/communication complexity of Algorithm~\ref{alg:mbdsvrg} for stochastic quadratic problems, and compare it with related methods in the literature. 
Throughout this section, we have $\ell(\w,\xi) = \frac{1}{2} (\w^\top \x - y)^2$ where $\xi = (\x, y)$. We assume that $\ell(\w,\xi)$ is $\beta$-smooth and $L$-Lipschitz in $\w$,\footnote{We can equivalently assume $\norm{\x}^2 \le \beta$ and $y$ is bounded.} 
and we would like to learn a predictor that is competitive to all predictors with norm at most $B$. Note that each $\ell(\w,\xi)$ is only weakly convex.

\subsection{Efficiency of MP-DSVRG}

For the distributed stochastic convex optimization problems, we are
concerned with efficiency in terms of sample, communication,
computation and memory. Recall that for convex $L$-Lipshitz,
$B$-bounded problems, % in order 
to learn a predictor $\hat \w$ with
$\varepsilon$-generalization error, \ie, $\mathbb{E} \left[ \phi(\hat
  \w) - \phi(\w_*) \right] \le \varepsilon$, we require the sample
size to be at least $n(\varepsilon) = \calO (L^2 B^2 /
\varepsilon^2)$. This sample complexity matches the worst case lower
bound, and can be achieved by vanilla SGD.

The theorem below shows that with careful choices of parameters in the outer and inner loops, MP-DSVRG achieves both communication and computation efficiency with the optimal sample complexity. 

\begin{thm}[Efficiency of MP-DSVRG]
  \label{thm:main-mbdsvrg}
  Set the parameters in Algorithm~\ref{alg:mbdsvrg} as follows:
  \begin{align*}
    \text{(outer loop)} & \qquad  T = \frac{n(\varepsilon)}{bm}, \quad \gamma = \frac{\sqrt{8 n(\varepsilon)}  L}{b m B}, \quad p_i = {\calO} \left( \frac{\sqrt{n(\varepsilon)}  L}{\beta m B} \right) \\
    \text{(inner loop)} & \qquad  K =  \calO \left( \log n(\varepsilon)  \right) .
  \end{align*}
  Then we have
  $
  \bbE \left[ \phi \left( \frac{1}{T} \sum_{t=1}^T \w_t \right) - \phi(\w_*) \right] \leq \frac{\sqrt{40} B L}{\sqrt{n(\varepsilon)}}  =  \calO \left( \varepsilon \right).
  $
  
  Moreover, Algorithm \ref{alg:mbdsvrg} can be implemented with 
  $
{\calO} \left( \frac{n(\varepsilon)}{bm} \log n (\varepsilon) \right)
  $
  rounds of communication, and each machine performs 
  $
{\calO} \left( \frac{n(\varepsilon)}{m} \log n (\varepsilon) \right)
  $
  vector operations in total.  % and each machine requires $\calO (b)$ memory, 
 % Here the notation $\tilde{\calO} (\cdot)$ hides poly-logarithmic dependences on $n (\varepsilon)$. 
\end{thm}

We comment on the choice of parameters. For sample efficiency, we fix the sample size $n (\varepsilon)$ and number of machines $m$, and so we can tradeoff the local minibatch size $b$ and the total number of outer iterations $T$, maintaining $b T = \frac{n (\varepsilon)}{m}$. For any $b$, the regularization parameters in the ``large minibatch'' problem is set to $\gamma = \sqrt{\frac{8T}{bm}} \cdot \frac{L}{B} = \frac{\sqrt{8 n (\varepsilon)} L}{bmB} $ according to Theorem~\ref{thm:inexact-weakly-convex-sufficient}. 
Moreover, we choose the number of batches $p_i$ in each local machine in a way that performing one pass of stochastic updates over a single batch by without-replacement sampling is sufficient to reduce the objective by a constant factor.

\section{Discussion and conclusion}

\begin{figure}[t]
  \centering
  \psfrag{Communication}[][]{Communication}
  \psfrag{Computation}[][]{Computation}
  \psfrag{Memory}[][]{Memory}
\psfrag{Stochastic DANE}[l][l]{MP-DANE}
\psfrag{Acc minibatch SGD}[l][l]{Acc. Minibatch SGD}
\psfrag{DiSCO/AIDE}[l][l]{DiSCO/AIDE}
\psfrag{DSVRG}[l][l]{DSVRG}
\psfrag{MP-DSVRG}[l][l]{MP-DSVRG}
\psfrag{MB-DSVRG}[l][l]{MP-DSVRG}
\psfrag{bacc}[][]{$b_{\operatorname{acc-sgd}}$}
\psfrag{bstar}[][]{$b_{\operatorname{mp-dane}}$}
\psfrag{bmax}[][]{$b_{max}$}
    \includegraphics[height=0.5\textheight, width=0.7\textwidth]{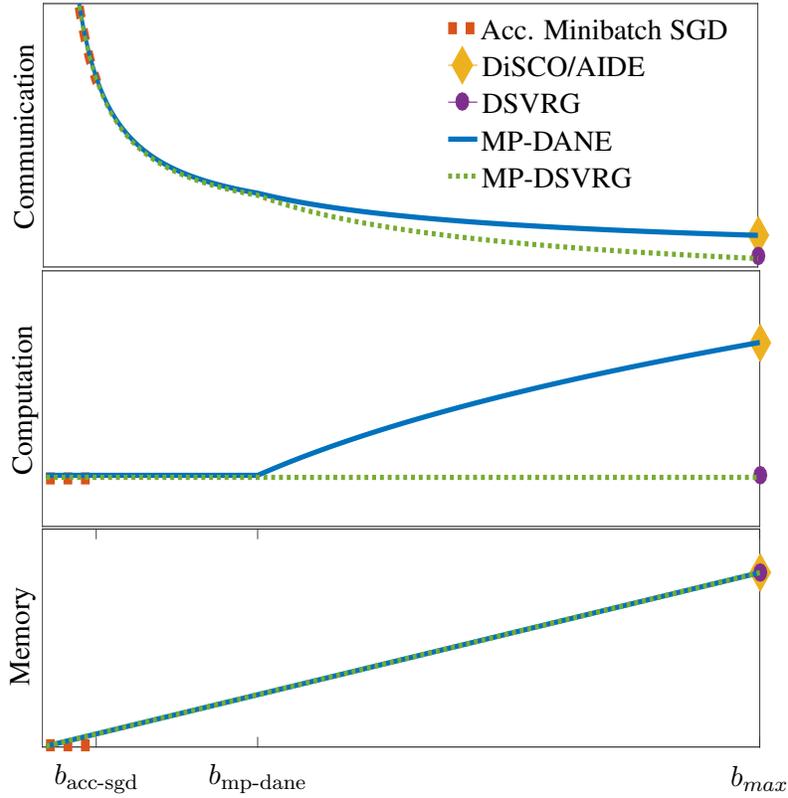}
        %  \makebox[1 \textwidth]{memory}
  \caption{Illustration of theoretical guarantees for MP-DSVRG and the comparison with accelerated minibatch SGD~\citep{Cotter_11a}, DiSCO~\citep{zhang2015disco}, AIDE~\citep{reddi2016aide}, DSVRG~\citep{lee2015distributed}, and MP-DANE (proposed and analyzed in Appendix~\ref{sec:stochastic-dane}). We plot the communication, computation and memory requirements while ensuring sample efficiency. Here $b_{\operatorname{acc-sgd}} \asymp n(\varepsilon)^{3/4}/(m\sqrt{B})$, $b_{\operatorname{mp-dane}} \asymp n(\varepsilon)/(m^2 B^2)$, and $b_{\max} = n(\varepsilon)/m$.}
  \label{fig:cartoon}
\end{figure}

In this paper, we made progress toward linear speedup, communication
and memory efficient methods for distributed stochastic optimization,
although we still do not have an algorithm that obtains the ``ideal''
distributed stochastic optimization performance of linear speedup with
constant or near-constant communication and memory.  There is also no
single known algorithm that dominates all others, with different
methods being preferable in terms of different resources.  These
tradeoffs, up to $\log$-factors, are given in
Table~\ref{table:summary} and the memory, communication and runtime
requirements are also schematically depicted in
Figure~\ref{fig:cartoon}.  In the figure, the horizontal axis
corresponds to the ``minibatch'' size, which can be controlled with
accelerated minibatch SGD and MP-DSVRG, while other methods are
batch methods which consider the entire data set.

From Figure~\ref{fig:cartoon} we can see that DSVRG (equivalent to
MP-DSVRG when $b = n(\varepsilon)/m$) dominates the other methods (up
to $\log$-factors) in terms of runtime and communication---it has
smaller communication requirements than DiSCO/AIDE (and better than
DANE) with nearly the same optimal runtime of accelerated minibatch
SGD.  But like other batch methods, it requires storing and
re-accessing the entire data set.  Accelerated minibatch SGD is the
only one of these methods requiring only $\calO(1)$ memory per
machine, and it achieves true linear speedup, but due to the limit on
the maximal allowed minibatch size, has relatively high communication
cost.  MP-DSVRG allows bridging these two extremes of memory and
communication, trading off between memory usage and communication.
The trade-off is almost an extrapolation, except that in the
low-memory high-communication extreme, MP-DSVRG still requires (small)
polynomial memory, not minibatch-SGD's $O(1)$ memory, and its runtime
still involve a logarithmic factor while minibatch-SGD achieves true
linear speedup.

Instead of using DSVRG to solve each proximal subproblem in a
minibatch-prox iteration, we can also use any other distributed
optimization approach.  For example, we can consider using DiSCO or
DANE.  This is depicted as ``MP-DANE'' in Figure
\ref{fig:cartoon}.  Again, an external minibatch-prox loop allows
trading off memory for communication.  For small minibatch sizes, up
to a critical value of $b_{\operatorname{mp-dane}} = \Theta( n(\varepsilon)/(m^2 B^2))$,
MP-DANE enjoys the same guarantees as MP-DSVRG.  But for
larger minibatch sizes, such an approach starts suffering from
DANE/DiSCO's inferior runtime and communication requirements compared
to DSVRG.

We emphasize that the above discussion is based on guarantees
established only for least square problems and ignores $\log$-factors.
We are unfortunately not aware of distributed stochastic optimization guarantees 
that improve over minibatch SGD (\ie,~achieve even near-linear speedup with lower
communication requirements) for general smooth objectives, or achieve
true linear speedup (and improved communication guarantees) even for
least-square problems.

\section*{Acknowledgement}
Research was partially supported by an Intel ICRI-CI award and NSF
awards IIS 1302662 and BIGDATA 1546500.  We would like to thank Ohad Shamir for discussions about
Distributed SVRG and Tong Zhang for discussions about minibatch-prox.

\bibliography{minibatch}

\appendix
\section{Analysis of exact minibatch-prox}

\subsection{Proof of Lemma~\ref{lem:exact-distance-to-any}}

\begin{proof}
Observe that~\eqref{e:exact-condition-t} implies $\gamma_t (\w_{t-1} - \w_t)$ is a subgradient at $\w_t$ of the sum of $\phit (\w)$ and the indicator function of $\Omega$ (which has value $0$ in $\Omega$ and $\infty$ otherwise), and thus we have for any $\w \in \Omega$ that
\begin{align} \label{e:exact-subgradient}
\phit (\w) - \phit (\w_t) \ge \gamma_t \left<\w_{t-1} - \w_t,\, \w - \w_t\right> + \frac{\lambda}{2} \norm{ \w - \w_t }^2.
\end{align}
For any $\w \in \Omega$, we can bound its distance to $\w_{t-1}$ as
\begin{align*}
\norm{\w_{t-1} - \w}^2 & = \norm{\w_{t-1} - \w_t + \w_t - \w}^2  \\
& = \norm{ \w_{t-1} - \w_t }^2 + 2 \left< \w_{t-1} - \w_{t},\, \w_t - \w \right> + \norm{ \w_{t} - \w }^2  \\
& \ge \norm{ \w_{t-1} - \w_t }^2 + \frac{2}{\gamma_t} \left( \phit (\w_t) - \phit (\w) \right)  + \frac{\lambda}{\gamma_t} \norm{ \w - \w_t }^2
+ \norm{ \w_t - \w }^2  \\ 
& = \frac{\lambda + \gamma_t}{\gamma_t} \norm{ \w_t - \w }^2
+ \frac{2}{\gamma_t} \left( \phit (\w_t) - \phit (\w) \right) + \norm{ \w_{t-1} - \w_t }^2
% & \ge \frac{\lambda + \gamma_t}{\gamma_t} \norm{ \w_t - \w }^2
% + \frac{2}{\gamma_t} \left( \phit (\w_t) - \phit (\w) \right)
\end{align*}
where we have used~\eqref{e:exact-subgradient} in the first inequality. 
Rearranging the terms yields the desired result.
\end{proof}

%% Removed the following result which is not essential.
%% 
% \subsection{Proof of Proposition~\ref{prop:exact-distance-reduction-without-stability}}

% \begin{proof}
% Take expectation of~\eqref{e:exact-distance-to-any} over the random sampling of $I_t$ and we obtain 
% \begin{align*}
%  \frac{\lambda + \gamma_t}{\gamma_t} \bbE_{I_t} \norm{ \w_t - \w }^2 & \le 
% \norm{\w_{t-1} - \w}^2 - \frac{2}{\gamma_t} \left( \bbE_{I_t} \left[ \phit (\w_t) \right] - \phi (\w) \right) \\
% & = \norm{\w_{t-1} - \w}^2 - \frac{2}{\gamma_t} \left( \bbE_{I_t} \left[ \phit (\w_t) - \phit (\w_{t-1}) \right] + \phi(\w_{t-1}) - \phi (\w) \right) \\
% & =  \norm{\w_{t-1} - \w}^2 - \frac{2}{\gamma_t} \left(\phi(\w_{t-1}) - \phi (\w) \right)
%  + \frac{2}{\gamma_t} \bbE_{I_t} \left[ \phit (\w_{t-1}) - \phit (\w_{t}) \right].
% \end{align*}
% Now, by the assumption that $\ell (\w,\xi)$ is $L$-Lipschitz, we have $\norm{ \nabla \ell (\w,\xi) } \le L$ and 
% \begin{align*}
% \phit (\w_{t-1}) - \phit (\w_{t}) & \le L \norm{ \w_{t-1} - \w_{t}  }
% = L \norm{ P_{\Omega} \left( \w_{t-1} \right) - P_{\Omega} \left( \w_{t-1} - \frac{1}{\gamma_t} \nabla \phit (\w_t) \right) } \\
% & \le \frac{L}{\gamma_t} \norm{  \nabla \phit (\w_t)  } \le \frac{L^2}{\gamma_t},
% \end{align*}
% where we have used~\eqref{e:exact-equivalent-sgd} in the first equality, and the nonexpansive property of the projection operator in the second inequality. Then the lemma follows.
% \end{proof}

\subsection{Proof of Lemma~\ref{lem:stability-at-time-t}}

The following lemma, which is essentially shown by~\citet[Theorem~6]{Shalev_09a}, characterizes the convergence of the empirical loss to the population counterpart for the (approximate) regularized empirical risk minimizer.

\begin{lem} \label{lem:stability}
Let the instantaneous function $\ell (\w, \xi)$ be $L$-Lipschitz and $\lambda$-strongly convex in $\w$. Consider the following regularized ERM problem with sample set $Z=\{ \xi_1,\dots,\xi_n \}$: 
\begin{align*}
\hw = \argmin_{\w \in \Omega} \ \hF (\w) \qquad \text{where}\quad \hF (\w) := \frac{1}{n} \sum_{i=1}^n \ell (\w, \xi_i)  + r (\w),
\end{align*}
and the regularizer $r (\w)$ is $\gamma$-strongly convex. 
Denote by $G (\w) = \bbE_{\xi} \left[\ell (\w, \xi) \right]$ and $\hG (\w) = \frac{1}{n} \sum_{i=1}^n \ell (\w, \xi_i)$ the expected and the empirical losses respectively. 
\begin{enumerate}
\item For the regularized empirical risk minimizer $\hw$,  we have 
\begin{align*}
\abs{ \bbE_{Z} \left[ G(\hw) - \hG(\hw) \right] } \le \frac{4 L^2}{(\lambda + \gamma) n}.
\end{align*}
\item If for any given dataset $Z$, a possibly randomized algorithm $\calA$ minimizes $\hF (\w)$ up to an error of $\eta$, \ie, $\calA$ returns an approximate solution $\tw$ such that
%\begin{align*}
$ \bbE_\calA \left[ \hF (\tw) -  \hF (\hw)\right] \le \eta $,
%\end{align*}
we have 
\begin{align*}
\abs{ \bbE_{Z,\calA} \left[ G(\tw) - \hG(\hw) \right] } \le \frac{4 L^2}{(\lambda + \gamma) n} + \sqrt{\frac{2 L^2 \eta}{\lambda + \gamma}}.
\end{align*}
\end{enumerate}
\end{lem}

\begin{proof}
We prove the lemma by a stability argument.

\paragraph{Exact ERM} Denote by $Z^{(i)}$ the sample set that is identical to $Z$ except that the $i$-th sample $\xi_i$ is replaced by another random sample $\xi_i^\prime$, by $\hF^{(i)} (\w)$ the empirical objective defined using $Z^{(i)}$, \ie,
\begin{align*}
\hF^{(i)} (\w) := \frac{1}{n} \left ( \sum_{j\neq i} \ell (\w, \xi_i) +  \ell (\w, \xi_i^\prime) \right) + r (\w),
\end{align*}
and by $\hw^{(i)}=\argmin_{\w\in \Omega}\ \hF^{(i)} (\w)$ the empirical risk minimizer of $\hF^{(i)} (\w)$.

By the definition of the empirical objectives, we have
\begin{align} 
\hF (\hw^{(i)}) - \hF (\hw) 
& = \frac{\ell (\hw^{(i)}, \xi_i) - \ell (\hw, \xi_i) }{n} 
+ \frac{\sum_{j\neq i} \ell (\hw^{(i)}, \xi_i) - \ell (\hw, \xi_i) }{n} 
+ r (\hw^{(i)}) - r (\hw) \nonumber \\
& = \frac{\ell (\hw^{(i)}, \xi_i) - \ell (\hw, \xi_i) }{n} 
+ \frac{ \ell (\hw, \xi_i^\prime) - \ell (\hw^{(i)}, \xi_i^\prime) }{n} 
+ \left( \hF^{(i)} (\hw^{(i)}) - \hF^{(i)} (\hw) \right) \nonumber  \\ 
& \le \frac{ \abs{\ell (\hw^{(i)}, \xi_i) - \ell (\hw, \xi_i) }}{n} 
+ \frac{ \abs{ \ell (\hw, \xi_i^\prime) - \ell (\hw^{(i)}, \xi_i^\prime) } }{n} \nonumber \\ \label{e:replacement-stability}
& \le \frac{2 L}{n} \norm{ \hw^{(i)} - \hw }
\end{align}
where we have used the fact that $\hw^{(i)}$ is the minimizer of $\hF^{(i)} (\w)$ in the first inequality, and the $L$-Lipschitz continuity of $\ell(\w,\xi)$ in the second inequality.

On the other hand, it follows from the $(\lambda + \gamma)$-strong convexity of $\hF (\w)$ that 
\begin{align} \label{e:stability-strong-convexity}
\hF (\hw^{(i)}) - \hF (\hw) \ge \frac{(\lambda + \gamma)}{2} \norm{ \hw^{(i)} - \hw }^2.
\end{align}

Combining~\eqref{e:replacement-stability} and~\eqref{e:stability-strong-convexity} yields $\norm{ \hw^{(i)} - \hw } \le \frac{4 L}{(\lambda + \gamma) n}$. 

Again, by the $L$-Lipschitz continuity of $\ell(\w,\xi)$, we have that for any sample $\xi$ that 
\begin{align} \label{e:generror-0}
\abs{ \ell(\hw,\xi) - \ell(\hw^{(i)},\xi) } \le L \norm{ \hw^{(i)} - \hw } \le \frac{4 L^2}{(\lambda + \gamma) n}.
\end{align}

Since $Z$ and $Z^{(i)}$ are both i.i.d. sample sets, we have
\begin{align*}
\bbE_Z \left[ G (\hw) \right] = \bbE_{Z^{(i)}} \left[ G (\hw^{(i)}) \right] 
= \bbE_{Z^{(i)} \cup \{\xi_i\}} \left[ \ell (\hw^{(i)}, \xi_i) \right].
\end{align*}
As this holds for all $i=1,\dots,n$, we can also write 
\begin{align} \label{e:generror-1}
\bbE_Z \left[ G (\hw) \right] 
= \frac{1}{n} \sum_{i=1}^n \bbE_{Z^{(i)} \cup \{\xi_i\}} \left[ \ell (\hw^{(i)}, \xi_i) \right].
\end{align}
On the other hand, we have
\begin{align} \label{e:generror-2}
\bbE_Z \left[ \hG (\hw) \right] 
=  \bbE_{Z} \left[ \frac{1}{n} \sum_{i=1}^n \ell (\hw, \xi_i) \right]
=  \frac{1}{n} \sum_{i=1}^n  \bbE_{Z} \left[ \ell (\hw, \xi_i) \right].
\end{align}
Combining~\eqref{e:generror-1} and~\eqref{e:generror-2} and using the stability~\eqref{e:generror-0}, we obtain
\begin{align*}
\bbE_Z \left[ G (\hw) - \hG (\hw) \right] = \frac{1}{n} \sum_{i=1}^n  \bbE_{Z \cup \{\xi_i^\prime\}} \left[ \ell (\hw^{(i)}, \xi_i) - \ell (\hw, \xi_i) \right] \in \left[- \frac{4 L^2}{(\lambda + \gamma) n},\, \frac{4 L^2}{(\lambda + \gamma) n} \right].
\end{align*}

\paragraph{Inexact ERM} For the approximate solution $\tw$, due to the $(\lambda + \gamma)$-strong convexity of $\hF (\w)$, we have
\begin{align*}
\bbE_\calA \norm{\tw - \hw}^2 \le \frac{2}{\lambda + \gamma} \bbE_\calA \left[ \hF(\tw) - \hF(\hw) \right] \le \frac{2 \eta}{\lambda + \gamma},
\end{align*}
and thus $\bbE_\calA  \norm{\tw - \hw} \le \sqrt{\frac{2 \eta}{\lambda + \gamma}}$ by the fact that $\bbE x^2 \ge \left( \bbE x \right)^2$ for any random variable $x$.

It then follows from the Lipschitz continuity of $G (\w)$ that
\begin{align*}
\bbE_\calA \abs{ G (\tw) - G (\hw) } \le L \cdot \bbE_\calA \norm{\tw - \hw} \le \sqrt{\frac{2 L^2 \eta}{\lambda + \gamma}}.
\end{align*}
Finally, we have by the triangle inequality and the stability of exact ERM that 
\begin{align*}
\abs{ \bbE_{Z,\calA} \left[ G(\tw) - \hG(\hw) \right] } 
& \le  \bbE_Z \left[ \bbE_{\calA} \abs{ G(\tw) - G(\hw) } \right]
+  \abs{ \bbE_Z \left[ G(\hw) - \hG(\hw) \right] } \\ 
& \le \sqrt{\frac{2 L^2 \eta}{\lambda + \gamma}} + \frac{4 L^2}{(\lambda + \gamma) n}.
\end{align*}
\end{proof}

Then Lemma~\ref{lem:stability-at-time-t} follows from the fact that that our stochastic objective~\eqref{e:exact-update-expt} is equipped with $L$-Lipschitz, $\lambda$-strongly convex loss $\phi(\w)$ and $\gamma_t$-strongly convex regularizer $\frac{\gamma_t}{2} \norm{\w - \w_{t-1}}^2$. 

\subsection{Proof of Lemma~\ref{lem:exact-distance-reduction-with-stability}}

\begin{proof}
We have by Lemma~\ref{lem:stability-at-time-t} that%%\footnote{Although Lemma~\ref{lem:stability} is stated for unconstrained objective, a convex constraint can be incorporated by adding the indicator function of the constraint set to the regularizer $r(\w)$ without changing its strong convexity, and the same stability result holds.}
\begin{align*}
\abs{ \bbE_{I_t} \left[ \phit (\w_t)  - \phi (\w_t) \right] } \le \frac{4 L^2}{ (\lambda + \gamma_t) b}.
\end{align*}
Take expectation of~\eqref{e:exact-distance-to-any} over the random sampling of $I_t$ and we obtain
\begin{align*}
\frac{\lambda + \gamma_t}{\gamma_t} \bbE_{I_t} \norm{ \w_t - \w }^2 
& \le \norm{\w_{t-1} - \w}^2 - \frac{2}{\gamma_t} \left( \bbE_{I_t} \left[ \phit (\w_t) \right] - \phi (\w) \right) \\
& = \norm{\w_{t-1} - \w}^2 - \frac{2}{\gamma_t} \left( \bbE_{I_t} \left[ \phit (\w_t)  - \phi (\w_t) \right] + \bbE_{I_t} \left[ \phi (\w_t) - \phi (\w) \right]\right)  \\
& \le \norm{\w_{t-1} - \w}^2 - \frac{2}{\gamma_t} \bbE_{I_t} \left[ \phi (\w_t) - \phi (\w) \right] + \frac{2}{\gamma_t} \abs{ \bbE_{I_t} \left[ \phit (\w_t)  - \phi (\w_t) \right] } \\
& \le \norm{\w_{t-1} - \w}^2 - \frac{2}{\gamma_t} \bbE_{I_t} \left[ \phi (\w_t) - \phi (\w) \right] + \frac{8 L^2}{\gamma_t (\lambda + \gamma_t) b}.
\end{align*}
\end{proof}

\subsection{Proof of Theorem~\ref{thm:rate-exact-minibatch-weakly-convex}}

\begin{proof}
When $\ell (\w,\xi)$ is weakly convex (\ie, $\lambda=0$), we further set $\gamma_t=\gamma$ for all $t\ge 1$. 
Applying Lemma~\ref{lem:exact-distance-reduction-with-stability} with $\w=\w_*$ yields 
\begin{align} \label{e:exact-minibatch-weakly-convex-distance-to-optimum}
\bbE_{I_t} \left[ \phi(\w_{t}) - \phi (\w_*) \right]  \le  \frac{\gamma}{2} \left( \norm{\w_{t-1} - \w_*}^2 - \bbE_{I_t} \norm{ \w_t - \w_* }^2 \right) + \frac{4 L^2}{\gamma b} .
\end{align}
Summing~\eqref{e:exact-minibatch-weakly-convex-distance-to-optimum} for $t=1,\dots,T$ yields 
\begin{align*}
\sum_{t=1}^T \bbE \left[ \phi(\w_{t}) - \phi (\w_*) \right] \le \frac{\gamma}{2}  \norm{\w_0 - \w_*}^2 + \frac{4 L^2 T}{\gamma b}.
\end{align*}
Minimizing the RHS over $\gamma$ gives the optimal choice
\begin{align*}
\gamma = \sqrt{\frac{8T}{b}} \cdot \frac{L}{\norm{\w_0 - \w_*}},
\end{align*}
with a corresponding regret
\begin{align*}
\frac{1}{T} \sum_{t=1}^T \bbE \left[ \phi(\w_{t}) - \phi (\w_*) \right] \le 
\frac{\sqrt{8} L}{\sqrt{bT}} \norm{\w_0 - \w_*}.
\end{align*}
As a result, by returning the uniform average $\widehat{\w}_T = \frac{1}{T} \sum_{t=1}^T \w_t$, we have due to the convexity of $\phi (\w)$ that
\begin{align*}
\bbE \left[ \phi(\widehat{\w}_T) - \phi (\w_*) \right] \le \frac{\sqrt{8} L}{\sqrt{bT}} \norm{\w_0 - \w_*}.
\end{align*}
\end{proof}

\subsection{Proof of Theorem~\ref{thm:rate-exact-minibatch-strongly-convex}}
\begin{proof}
Let $\ell (\w,\xi)$ be $\lambda$-strongly convex for some $\lambda>0$. 
Applying Lemma~\ref{lem:exact-distance-reduction-with-stability} with $\w=\w_*$ yields 
\begin{align} \label{e:exact-minibatch-strongly-convex-distance-to-optimum}
\bbE_{I_t} \left[ \phi(\w_{t}) - \phi (\w_*) \right]  \le \left( \frac{\gamma_t}{2}  \norm{\w_{t-1} - \w_*}^2 - \frac{\lambda + \gamma_t}{2} \bbE_{I_t} \norm{ \w_t - \w_* }^2 \right) + \frac{4 L^2}{ (\lambda + \gamma_t) b} .
\end{align}
Setting $\gamma_t = \frac{\lambda (t-1)}{2}$ for $t=1,\dots,$\footnote{This choice is inspired by the stepsize rule of~\citet{Lacost_12a} for stochastic gradient descent.}, the above inequality becomes
\begin{align*} 
\bbE_{I_t} \left[ \phi(\w_{t}) - \phi (\w_*) \right]  
& \le  \left(  \frac{\lambda (t-1)}{4} \norm{\w_{t-1} - \w_*}^2 - \frac{\lambda (t+1)}{4} \bbE_{I_t} \norm{ \w_t - \w_* }^2 \right) + \frac{8 L^2}{ \lambda b (t+1)} \\
& \le  \left( \frac{\lambda (t-1)}{4} \norm{\w_{t-1} - \w_*}^2 - \frac{\lambda (t+1)}{4} \bbE_{I_t} \norm{ \w_t - \w_* }^2 \right) + \frac{8 L^2}{ \lambda b t},
\end{align*}
and therefore
\begin{align*} 
t \cdot \bbE_{I_t} \left[ \phi(\w_{t}) - \phi (\w_*) \right]  
& \le  \frac{\lambda}{4} \left( (t-1) t \norm{\w_{t-1} - \w_*}^2 - t (t+1) \bbE_{I_t} \norm{ \w_t - \w_* }^2 \right) + \frac{8 L^2}{ \lambda b} .
\end{align*}
Summing this inequality for $t=1,\dots,T$ yields
\begin{align*} 
\sum_{t=1}^T t \cdot \bbE \left[ \phi(\w_{t}) - \phi (\w_*) \right]   \le  \frac{8 L^2 T}{\lambda b} .
\end{align*}
As a result, by returning the weighted average $\widehat{\w}_T = \frac{2}{T (T+1)} \sum_{t=1}^T t \w_t$, we have due to the convexity of $\phi (\w)$ that 
$\phi (\widehat{\w}_T) \le \frac{2}{T (T+1)} \sum_{t=1}^T t \cdot \phi (\w_t)$ and 
\begin{align*}
\bbE \left[ \phi(\widehat{\w}_T) - \phi (\w_*) \right] 
\le  \frac{2}{T (T+1)} \sum_{t=1}^T t \cdot \bbE \left[ \phi(\w_{t}) - \phi (\w_*) \right] 
\le \frac{16 L^2 }{\lambda b (T+1)}.
\end{align*}
\end{proof}

\section{Analysis of inexact minibatch-prox}

\subsection{Proof of Lemma~\ref{lem:inexact-distance-to-any}}

\begin{proof}
  Due to the $(\lambda + \gamma_t)$-strong convexity of $\tilde{f}_t (\w)$, we have
  \begin{align*}
    \bbE_{\calA} \norm{ \tw_t - \bw_t }^2 \le \frac{2}{\lambda + \gamma_t} \bbE_{\calA} \left[ \tilde{f}_t (\tw_t) - \tilde{f}_t (\bw_t) \right] 
    \le \frac{2 \eta_t}{\lambda + \gamma_t}.
  \end{align*}

  Applying Lemma~\ref{lem:exact-distance-to-any} to the exact minimizer $\bw_t$ yields 
  \begin{align*}
    \phit (\bw_t) - \phit (\w) \le \frac{\gamma_t}{2} \norm{\tw_{t-1} - \w}^2 - \frac{\lambda + \gamma_t}{2} \norm{ \bw_t - \w }^2 .
  \end{align*}
  Therefore, for the $t$-th iteration, we have 
  \begin{align*}
    & \bbE_{I_t,\calA} \left[ \phi (\tw_t) - \phi (\w) \right] \\
    = \  & \bbE_{I_t,\calA} \left[ \phi (\tw_t) - \phit (\bw_t) \right] + \bbE_{I_t} \left[ \phit (\bw_t) - \phit (\w) \right] \\
    \le \  & \frac{4 L^2}{(\lambda + \gamma_t) b} + \sqrt{\frac{2 L^2 \eta_t}{\lambda + \gamma_t}} + \frac{\gamma_t}{2} \norm{\tw_{t-1} - \w}^2 - \frac{\lambda + \gamma_t}{2} \bbE_{I_t} \norm{ \bw_t - \w }^2  \\
    \le \  & \frac{4 L^2}{(\lambda + \gamma_t) b} + \sqrt{\frac{2 L^2 \eta_t}{\lambda + \gamma_t}} + \frac{\gamma_t}{2} \norm{\tw_{t-1} - \w}^2 - \frac{\lambda + \gamma_t}{2}  \bbE_{I_t,\calA} \left( \norm{ \tw_t - \w } - \norm{\tw_t - \bw_t} \right)^2  \\
    \le \  & \frac{4 L^2}{(\lambda + \gamma_t) b} + \sqrt{\frac{2 L^2 \eta_t}{\lambda + \gamma_t}} + \frac{\gamma_t}{2} \norm{\tw_{t-1} - \w}^2 - \frac{\lambda + \gamma_t}{2} \bbE_{I_t,\calA} \norm{ \tw_t - \w }^2 \\
    & \qquad\qquad\qquad\qquad\qquad\qquad\qquad\qquad + (\lambda + \gamma_t) \cdot \bbE_{I_t,\calA} \left[ \norm{\tw_t - \bw_t} \cdot \norm{ \tw_t - \w } \right] \\
    \le \  & \frac{4 L^2}{ (\lambda + \gamma_t) b} + \sqrt{\frac{2 L^2 \eta_t}{\lambda + \gamma_t}} + \frac{\gamma_t}{2} \norm{\tw_{t-1} - \w}^2 - \frac{\lambda + \gamma_t}{2} \bbE_{I_t,\calA} \norm{ \tw_t - \w }^2 \\
    & \qquad\qquad\qquad\qquad\qquad\qquad\qquad\qquad + (\lambda + \gamma_t)  \sqrt{ \bbE_{I_t,\calA} \norm{\tw_t - \bw_t}^2 } \cdot \sqrt{ \bbE_{I_t,\calA} \norm{ \tw_t - \w }^2 } \\
    \le \  & \frac{4 L^2}{(\lambda + \gamma_t) b} + \sqrt{\frac{2 L^2 \eta_t}{\lambda + \gamma_t}} + \frac{\gamma_t}{2} \norm{\tw_{t-1} - \w}^2 - \frac{\lambda + \gamma_t}{2} \bbE_{I_t,\calA} \norm{ \tw_t - \w }^2 \\
    & \qquad\qquad\qquad\qquad\qquad\qquad\qquad\qquad +  \sqrt{2 (\lambda + \gamma_t) \eta_t} \cdot \sqrt{ \bbE_{I_t,\calA} \norm{ \tw_t - \w }^2 }
  \end{align*}
  where we have applied Lemma~\ref{lem:stability} to the approximate minimizer $\tw_t$ in the first inequality, used the triangle inequality $\norm{ \bw_t - \w } \ge \abs{ \norm{ \tw_t - \w } - \norm{\tw_t - \bw_t} }$ in the second inequality, dropped a negative term in the third inequality, and used the Cauchy-Schwarz inequality for random variables in the fourth inequality.
\end{proof}

\subsection{Proof of Theorem~\ref{thm:inexact-weakly-convex-sufficient}}

When $\ell (\w,\xi)$ is weakly convex (\ie, $\lambda=0$), set $\gamma_t = \gamma$ for all $t\ge 1$ as in exact minibatch-prox. Then summing~\eqref{e:inexact-distance-to-any} for $t=1,\dots,T$ yields 
\begin{align} 
  \sum_{t=1}^T \bbE \left[ \phi (\tw_t) - \phi (\w_*) \right] + \frac{\gamma}{2} \bbE \norm{\tw_T - \w_*}^2 
  & \le \frac{\gamma}{2} \norm{\tw_0 - \w_*}^2 +   \frac{4 L^2 T}{\gamma b}  + \sum_{t=1}^T \sqrt{\frac{2 L^2 \eta_t}{\gamma}} \nonumber \\ \label{e:inexact-recursion-weakly-convex}
  &\qquad + \sum_{t=1}^T \sqrt{2 \gamma \eta_t} \cdot \sqrt{ \bbE \norm{ \tw_t - \w_*}^2 }
\end{align}
where the expectation is taken over random sampling and the randomness of $\calA$ in the first $T$ iterations. 
To resolve the recursion, we need the following lemma by~\citet{Schmid_11a}. 
\begin{lem}\label{lem:resolve-recursion}
  Assume that the non-negative sequence $\{u_T\}$ satisfies the following recursion for all $T \ge 1$:
  \begin{align*}
    u_T^2 \le S_T + \sum_{t=1}^T \lambda_t u_t,
  \end{align*}
  with $S_T$ an increasing sequence, $S_0 \ge u_0^2$ and $\lambda_t \ge 0$ for all $t$. Then, for all $T\ge 1$, we have % \footnote{The original lemma of~\citet{Schmid_11a} has only the first inequality; we used the fact that $\sqrt{x+y}\le \sqrt{x} + \sqrt{y}$ for $x,\, y\ge 0$ to obtain the second inequality.}
  \begin{align*}
    u_T \le \frac{1}{2} \sum_{t=1}^T \lambda_t + \left( S_T + \left(\frac{1}{2} \sum_{t=1}^T \lambda_t \right)^2 \right)^{\frac{1}{2}} \le \sqrt{S_T} + \sum_{t=1}^T \lambda_t.
  \end{align*}
\end{lem}

We are now ready to prove Theorem~\ref{thm:inexact-weakly-convex-sufficient}.
\begin{proof}
  \textbf{Bounding $\sqrt{ \bbE \norm{ \tw_t - \w_* }^2 }$.} 
  Dropping the $\sum_{t=1}^T \bbE \left[ \phi (\tw_t) - \phi (\w_*) \right]$ term from~\eqref{e:inexact-recursion-weakly-convex} which is non-negative due to the optimality of $\w_*$, we obtain
  \begin{align*}
    \bbE \norm{\tw_T - \w_*}^2 \le \norm{\tw_0 - \w_*}^2 +  \frac{8 L^2 T}{\gamma^2 b} + \sum_{t=1}^T \sqrt{\frac{8 L^2 \eta_t}{\gamma^3}} +  \sum_{t=1}^T \sqrt{\frac{8 \eta_t }{\gamma}} \cdot \sqrt{ \bbE \norm{ \tw_t - \w_*}^2 }.
  \end{align*}
  Now apply Lemma~\ref{lem:resolve-recursion} (using $u_T=\sqrt{\bbE \norm{\tw_T - \w_*}^2}$, $S_T=\norm{\tw_0 - \w_*}^2 + \frac{8 L^2 T}{\gamma^2 b} + \sum_{t=1}^T \sqrt{\frac{8 L^2 \eta_t}{\gamma^3}}$, and $\lambda_t = \sqrt{\frac{8 \eta_t }{\gamma}}$) 
  and the fact that $\sqrt{x+y}\le \sqrt{x} + \sqrt{y}$ for $x,\, y\ge 0$, we have
  \begin{align*}
    \sqrt{ \bbE \norm{\tw_T - \w_*}^2 } 
    & \le \norm{\tw_0 - \w_*} + \sqrt{ \frac{8 L^2 T}{\gamma^2 b}} 
    + \sum_{t=1}^T \sqrt{\frac{8 \eta_t }{\gamma}}  + \sqrt{ \sum_{t=1}^T \sqrt{\frac{8 L^2 \eta_t}{\gamma^3}} }
  \end{align*}
  We have thus bounded the sequence of $\sqrt{ \bbE \norm{\tw_T - \w_*}^2 }$ by a non-negative increasing sequence. 

  \textbf{Bounding function values.} 
  Dropping the $\bbE \norm{\tw_T - \w_*}^2$ term from~\eqref{e:inexact-recursion-weakly-convex} which is non-negative, we obtain
  \begin{align} 
    & \sum_{t=1}^T \bbE \left[ \phi (\tw_t) - \phi (\w_*) \right] \nonumber \\
    \le\ & \frac{\gamma}{2} \norm{\tw_0 - \w_*}^2 + \frac{4 L^2 T}{\gamma b} +  \sum_{t=1}^T  \sqrt{\frac{2 L^2 \eta_t}{\gamma}} +  \sum_{t=1}^T \sqrt{2 \eta_t \gamma} \cdot \sqrt{ \bbE \norm{ \tw_t - \w_*}^2 } \nonumber \\
    \le\ & \frac{\gamma}{2} \norm{\tw_0 - \w_*}^2 + \frac{4 L^2 T}{\gamma b} +  \sum_{t=1}^T  \sqrt{\frac{2 L^2 \eta_t}{\gamma}} + \left( \sum_{t=1}^T \sqrt{2 \eta_t \gamma} \right) \cdot \max_{1\le t\le T} \sqrt{ \bbE \norm{ \tw_t - \w_*}^2 } \nonumber \\ 
    \le\ & \frac{\gamma}{2} \norm{\tw_0 - \w_*}^2 +  \frac{4 L^2 T}{\gamma b} +  \sum_{t=1}^T  \sqrt{\frac{2 L^2 \eta_t}{\gamma}} \nonumber \\ \label{e:inexact-weakly-convex-function-values}
    & \quad + \left( \sum_{t=1}^T \sqrt{2 \eta_t \gamma} \right) \cdot \left( \norm{\tw_0 - \w_*} + \sqrt{ \frac{8 L^2 T}{\gamma^2 b} } + \sum_{t=1}^T \sqrt{\frac{8 \eta_t }{\gamma}} +  \sqrt{ \sum_{t=1}^T \sqrt{\frac{8 L^2 \eta_t}{\gamma^3}} } \right) .
  \end{align}

  To achieve the same order of regret as in exact minibatch-prox, we require that $\eta_t$ decays with $t$, and in particular 
  \begin{align} \label{e:inexact-weakly-convex-eta_t}
    \eta_t \le \min \left( c_1 \left( \frac{T}{b} \right)^{\frac{1}{2}}, \, c_2 \left( \frac{T}{b} \right)^{\frac{3}{2}}  \right) 
    \cdot \frac{ L \norm{\tw_0 - \w_*}}{t^{2+2\delta}}
  \end{align}
  for some $\delta>0$. Note that $\eta_t$ has the unit of function value. Let $c:=\sum_{i=1}^{\infty} \frac{1}{i^{1+\delta}} \le \frac{1+\delta}{\delta}$ which only depends on $\delta$ (as a concrete example, we have $c = \frac{\pi^2}{6}$ when $\delta=2$). 

  Using the choice of $\gamma=\sqrt{\frac{8T}{b}} \cdot \frac{L}{\norm{\w_0 - \w_*}}$, we obtain from~\eqref{e:inexact-weakly-convex-eta_t} that 
  \begin{align*}
    \sum_{t=1}^T \sqrt{\frac{8 \eta_t }{\gamma}} 
    & = \sum_{t=1}^T \sqrt{ \sqrt{\frac{8b}{T}} \cdot \frac{\norm{\w_0 - \w_*}} {L} \cdot \eta_t } 
    \le 8^{\frac{1}{4}} c_1^{\frac{1}{2}} \norm{\tw_0 - \w_*} \sum_{t=1}^T \frac{1}{t^{1+\delta}} \\
    & \le 8^{\frac{1}{4}} c_1^{\frac{1}{2}} c \norm{\tw_0 - \w_*}, \\
    \sum_{t=1}^T \sqrt{\frac{8 L^2 \eta_t }{\gamma^3}} 
    & = \sum_{t=1}^T \sqrt{ \sqrt{\frac{b^3}{8 T^3}} \cdot \frac{\norm{\w_0 - \w_*}^3} {L} \cdot \eta_t } 
    \le  8^{-\frac{1}{4}} c_2^{\frac{1}{2}} \norm{\tw_0 - \w_*}^2 \sum_{t=1}^T \frac{1}{t^{1+\delta}} \\
    & \le  8^{-\frac{1}{4}} c_2^{\frac{1}{2}} c \norm{\tw_0 - \w_*}^2.
  \end{align*}

  Continuing from~\eqref{e:inexact-weakly-convex-function-values} and substituting in the value of $\gamma$, we have 
  \begin{align*} 
    \sum_{t=1}^T \bbE \left[ \phi (\tw_t) - \phi (\w_*) \right] & 
    \le \sqrt{\frac{8 T}{b}} \cdot L \norm{\tw_0 - \w_*}  + \frac{\gamma}{2} \sum_{t=1}^T \sqrt{\frac{8 L^2 \eta_t }{\gamma^3}} \\ 
    & \quad +  \frac{\gamma}{2} \left( \sum_{t=1}^T \sqrt{\frac{8 \eta_t }{\gamma}} \right) \cdot \left( 2 \norm{\tw_0 - \w_*} + \sum_{t=1}^T \sqrt{\frac{8 \eta_t }{\gamma}} + \sqrt{ \sum_{t=1}^T \sqrt{\frac{8 L^2 \eta_t}{\gamma^3}} } \right) \\
    & = \sqrt{\frac{8 T}{b}} \cdot L \norm{\tw_0 - \w_*} + \sqrt{\frac{2T}{b}} \cdot \frac{L}{\norm{\w_0 - \w_*}} \cdot 8^{-\frac{1}{4}} c_2^{\frac{1}{2}} c \norm{\tw_0 - \w_*}^2 \\
    & \qquad + \sqrt{\frac{2T}{b}} \cdot \frac{L}{\norm{\w_0 - \w_*}} \cdot 8^{\frac{1}{4}} c_1^{\frac{1}{2}} c \norm{\tw_0 - \w_*} \times  \\
    & \qquad \left( 2 \norm{\tw_0 - \w_*} + 8^{\frac{1}{4}} c_1^{\frac{1}{2}} c \norm{\tw_0 - \w_*} + \sqrt{ 8^{-\frac{1}{4}} c_2^{\frac{1}{2}} c \norm{\tw_0 - \w_*}^2  }  \right) \\
    & = c_3 \sqrt{\frac{T}{b}} \cdot L \norm{\tw_0 - \w_*}.
  \end{align*}
%   Moreover, when $\delta = 1/2$, $c_1 = c_2 = 10^{-4}$, since $c = 3$, and $ \norm{\tw_0 - \w_*} = \norm{\w_0 - \w_*}$, we have
%   \begin{align*}
%   \sum_{t=1}^T \bbE \left[ \phi (\tw_t) - \phi (\w_*) \right] \leq& 
%   \left( \sqrt{8} + \sqrt{2} \cdot 8^{-\frac{1}{4}} \cdot 3/100 + \sqrt{2} \cdot 8^{\frac{1}{4}} \cdot 3/100 \cdot \left(2 + 8^{\frac{1}{4}} \cdot 3/100 + \sqrt{8^{-\frac{1}{4}} \cdot 3/100} \right) \right) \\
%   &\cdot \frac{\sqrt{T} L \norm{\w_0 - \w_*}}{b} \\
%   \leq& \frac{\sqrt{10 T} L \norm{\w_0 - \w_*}}{b}
%   \end{align*}
  The suboptimality of $\widehat{\w}_T$ is then due to the convexity of $\phi (\w)$:
  \begin{align*}
    \bbE \left[ \phi(\widehat{\w}_T) - \phi (\w_*) \right] \le \frac{1}{T} \sum_{t=1}^T \bbE \left[ \phi (\tw_t) - \phi (\w_*) \right]
    =  \frac{c_3 L \norm{\tw_0 - \w_*}}{\sqrt{bT}} .
  \end{align*}
\end{proof}

\subsection{Proof of Theorem~\ref{thm:inexact-strongly-convex-sufficient}}

\begin{proof}
  We have by Lemma~\ref{lem:inexact-distance-to-any} that 
  \begin{align*} 
    \bbE_{I_t,\calA} \left[ \phi (\tw_t) - \phi (\w_*) \right]  
    & \le \frac{\lambda (t-1)}{4} \norm{\tw_{t-1} - \w_*}^2 
    - \frac{\lambda (t+1)}{4} \bbE_{I_t,\calA} \norm{\tw_t - \w_*}^2 
    \\
    & \qquad\qquad + \frac{8 L^2}{\lambda b (t+1)} + \sqrt{\frac{4 L^2 \eta_t}{\lambda (t+1)}} + \sqrt{\lambda (t+1) \eta_t} \cdot \sqrt{ \bbE_{I_t,\calA} \norm{ \tw_t - \w_*}^2 }.
  \end{align*}
  Relaxing the $\frac{1}{t+1}$ to $\frac{1}{t}$ on the RHS, and multiplying both sides by $t$, we further obtain
  \begin{align*} 
    t \cdot \bbE_{I_t,\calA} \left[ \phi (\tw_t) - \phi (\w_*) \right]  
    & \le \frac{\lambda (t-1) t}{4} \norm{\tw_{t-1} - \w_*}^2 
    - \frac{\lambda t (t+1)}{4} \bbE_{I_t,\calA} \norm{\tw_t - \w_*}^2   \\
    & \qquad\qquad + \frac{8 L^2}{\lambda b} + \sqrt{\frac{4 L^2 t \eta_t}{\lambda}} + \sqrt{\lambda t \eta_t} \cdot \sqrt{ \bbE_{I_t,\calA} \left[ t (t+1) \norm{ \tw_t - \w_*}^2 \right] }.
  \end{align*}

  Summing this inequality for $t=1,\dots,T$ yields 
  \begin{gather} 
    \sum_{t=1}^T t \cdot \bbE \left[ \phi (\tw_t) - \phi (\w_*) \right] + \frac{\lambda T (T+1)}{4} \bbE \norm{\tw_T - \w_*}^2   \nonumber \\ \label{e:inexact-recursion-strongly-convex}
    \le \frac{8 L^2 T}{\lambda b} + \sum_{t=1}^T \sqrt{\frac{4 L^2 t \eta_t}{\lambda}}
    + \sum_{t=1}^T \sqrt{\lambda t \eta_t} \cdot \sqrt{ \bbE \left[ t (t+1) \norm{ \tw_t - \w_*}^2 \right] }.
  \end{gather}

  \textbf{Bounding $\sqrt{ \bbE \norm{ \tw_t - \w_* }^2 }$.} 
  Dropping the $\sum_{t=1}^T  t \cdot \bbE \left[ \phi (\tw_t) - \phi (\w_*) \right]$ term from~\eqref{e:inexact-recursion-strongly-convex} which is non-negative due to the optimality of $\w_*$, we obtain
  \begin{align*} 
    \bbE \left[ T (T+1) \norm{\tw_T - \w_*}^2 \right]
    \le & \frac{32 L^2 T}{\lambda^2 b} + \sum_{t=1}^T \sqrt{\frac{64 L^2 t \eta_t}{\lambda^3}}
    + \sum_{t=1}^T \sqrt{\frac{16 t \eta_t}{\lambda}} \cdot \sqrt{ \bbE \left[ t (t+1) \norm{ \tw_t - \w_*}^2 \right] }.
  \end{align*}
  Applying Lemma~\ref{lem:resolve-recursion} (using $u_T=\sqrt{\bbE \left[ T (T+1) \norm{\tw_T - \w_*}^2 \right]}$, $S_T=\frac{32 L^2 T}{\lambda^2 b} + \sum_{t=1}^T \sqrt{\frac{64 L^2 t \eta_t}{\lambda^3}}$, and $\lambda_t = \sqrt{\frac{16 t \eta_t}{\lambda}}$), we have 
  \begin{align*} 
    \sqrt{\bbE \left[ T (T+1) \norm{\tw_T - \w_*}^2 \right]} \le \sqrt{\frac{32 L^2 T}{\lambda^2 b}} +  \sum_{t=1}^T \sqrt{\frac{16 t \eta_t}{\lambda}} + \sqrt{\sum_{t=1}^T \sqrt{\frac{64 L^2 t \eta_t}{\lambda^3}}} .
  \end{align*}

  \textbf{Bounding function values.} 
  Dropping the $\bbE \norm{\tw_T - \w_*}^2$ term from~\eqref{e:inexact-recursion-strongly-convex} which is non-negative, we obtain
  \begin{align} 
    \sum_{t=1}^T t \cdot \bbE \left[ \phi (\tw_t) - \phi (\w_*) \right] & 
    \le \frac{8 L^2 T}{\lambda b} + \sum_{t=1}^T \sqrt{\frac{4 L^2 t \eta_t}{\lambda}} 
    \nonumber \\ \label{e:inexact-strongly-convex-function-values}
    & \hspace*{-4em} + \left( \sum_{t=1}^T \sqrt{\lambda t \eta_t} \right) \cdot \left( \sqrt{\frac{32 L^2 T}{\lambda^2 b}} +  \sum_{t=1}^T \sqrt{ \frac{16 t \eta_t}{\lambda} } + \sqrt{\sum_{t=1}^T \sqrt{\frac{64 L^2 t \eta_t}{\lambda^3}}} \right).
  \end{align}

  To achieve the same order of regret as in exact minibatch-prox, we require that $\eta_t$ decays with $t$, and in particular 
  \begin{align} \label{e:inexact-strongly-convex-eta_t}
    \eta_t \le \min \left( c_1 \left( \frac{T}{b} \right),\, c_2 \left( \frac{T}{b} \right)^2 \right) \cdot  \frac{L^2 }{t^{3+2\delta} \lambda } 
  \end{align}
  for some $\delta>0$. Note that $\eta_t$ has the unit of function value. Let $c:=\sum_{i=1}^{\infty} \frac{1}{i^{1+\delta}} \le \frac{1+\delta}{\delta}$. Then~\eqref{e:inexact-strongly-convex-eta_t} ensures that 
  \begin{align*}
    \sum_{t=1}^T \sqrt{\frac{t \eta_t }{\lambda}} \le c \sqrt{c_1} \sqrt{\frac{L^2 T}{\lambda^2 b}}, \qquad \text{and} \qquad 
    \sum_{t=1}^T \sqrt{\frac{L^2 t \eta_t }{\lambda^3}} \le c \sqrt{c_2} \cdot \frac{L^2 T}{\lambda^2 b}.
  \end{align*}
  Continuing from~\eqref{e:inexact-strongly-convex-function-values}, we have
  \begin{align*} 
    \sum_{t=1}^T t \cdot \bbE \left[ \phi (\tw_t) - \phi (\w_*) \right] 
    & \le \frac{8 L^2 T}{\lambda b}  + 2 c \sqrt{c_2} \cdot \frac{L^2 T}{\lambda b} \\
    &\quad + c \sqrt{c_1} \sqrt{\frac{L^2 T}{b}} \left( \sqrt{\frac{32 L^2 T}{\lambda^2 b}} + 4 c \sqrt{c_1} \sqrt{\frac{L^2 T}{\lambda^2 b}} + \sqrt[4]{64 c^2 c_2} \sqrt{\frac{L^2 T}{\lambda^2 b}} \right) \\
    & = \frac{c_3}{2} \cdot \frac{L^2 T}{\lambda b}.
  \end{align*}
%   Moreover, if we choose $\delta = 1/2$, $c_1 = c_2 = 10^{-4}$, we have $c = 3$, thus
%   \begin{align*} 
%   \sum_{t=1}^T t \cdot \bbE \left[ \phi (\tw_t) - \phi (\w_*) \right] \leq& 
%   \left( 8 + 6/100 + 3/100 \left( \sqrt{32} + 12/100 + \sqrt[4]{64*9/100} \right) \right) \cdot \frac{L^2 T}{\lambda b} \\
%   \leq& \frac{8.25 L^2 T}{\lambda b}
%   \end{align*}
  In view of the convexity of $\phi(\w)$, by returning the weighted average $\widehat{\w}_T = \frac{2}{T (T+1)} \sum_{t=1}^T t \tw_t$, we have 
  \begin{align*}
    \bbE \left[ \phi(\widehat{\w}_T) - \phi (\w_*) \right] 
    \le \frac{2}{T (T+1)} \sum_{t=1}^T t \cdot \bbE \left[ \phi (\tw_t) - \phi (\w_*) \right] 
    \le \frac{c_3 L^2}{\lambda b (T+1)}.
  \end{align*}
\end{proof}

\subsection{Connection to minibatch stochastic gradient descent}
\label{sec:minibatch-SGD}
To see the connection between minibatch-prox and minibatch SGD, note that if we solve the linearized minibatch problem exactly, we obtain the minibatch stochastic gradient descent algorithm: 
\begin{align*}
  \tw_t = \argmin_{\w \in \Omega}\; \phit (\tw_{t-1}) + \nabla \left< \phit (\tw_{t-1}),\, \w - \tw_{t-1} \right> + \frac{\gamma_t}{2} \norm{\w-\tw_{t-1}}^2 .
\end{align*}

Following~\citet{Cotter_11a}, we assume that $\ell(\w,\xi)$ is $\beta$-smooth: 
\begin{align*}
  \norm{ \nabla \ell(\w,\xi) - \nabla \ell(\w^\prime,\xi) } \le \beta \norm{\w - \w^\prime},
  \qquad \forall \w, \w^\prime \in \Omega.
\end{align*}
%% Assume that the domain $\Omega$ is bounded with radius $B=\sup_{\w \in \Omega}\, \norm{\w}$. 
We then have the following guarantee for each iterate of minbatch SGD.

\begin{prop}\label{prop:minibatch-SGD-time-t}
  For iteration $t$ of minibatch SGD, we have
  \begin{align} \label{e:minibatch-sgd-time-t}
    \bbE_{I_t} \left[ \phi (\tw_{t}) - \phi (\w_*) \right] \le 
    \frac{2 L^2}{(\gamma_t - \beta) b} + \frac{\gamma_t - \lambda}{2} \norm{\w_* - \tw_{t-1}}^2 - \frac{\gamma_t}{2} \bbE_{I_t} \norm{\w_* - \tw_{t}}^2 .
  \end{align}
\end{prop}
\begin{proof} Our proof closely follows that of~\citet{Cotter_11a}. 

  Due to the smoothness of $\phi$, we have that
  \begin{align}
    \phi (\tw_{t}) & \le \phi (\tw_{t-1}) + \left< \nabla \phi (\tw_{t-1}),\, \tw_{t} - \tw_{t-1} \right> + \frac{\beta}{2} \norm{\tw_{t} - \tw_{t-1}}^2 \nonumber \\
    & \le \phi (\tw_{t-1}) + \left< \nabla \phi (\tw_{t-1}) - \nabla \phit (\tw_{t-1}),\, \tw_{t} - \tw_{t-1} \right> + \frac{\beta}{2} \norm{\tw_{t} - \tw_{t-1}}^2 \nonumber \\
    & \qquad\qquad\qquad\qquad\qquad\qquad\qquad\qquad 
    + \left< \nabla \phit (\tw_{t-1}),\, \tw_{t} - \tw_{t-1} \right> \nonumber \\
    & =  \phi (\tw_{t-1}) + \norm{ \nabla \phi (\tw_{t-1}) - \nabla \phit (\tw_{t-1}) } \cdot \norm{ \tw_{t} - \tw_{t-1} } + \frac{\beta}{2} \norm{\tw_{t} - \tw_{t-1}}^2 \nonumber \\ 
    & \qquad\qquad\qquad\qquad\qquad\qquad\qquad\qquad 
    + \left< \nabla \phit (\tw_{t-1}),\, \tw_{t} - \tw_{t-1} \right> \nonumber \\
    & \le   \phi (\tw_{t-1}) + \frac{1}{2 (\gamma_t - \beta)} \norm{ \nabla \phi (\tw_{t-1}) - \nabla \phit (\tw_{t-1}) }^2 + \frac{\gamma_t - \beta}{2} \norm{ \tw_{t} - \tw_{t-1} }^2 \nonumber \\
    & \qquad\qquad\qquad\qquad\qquad\qquad\qquad\qquad
    + \frac{\beta}{2} \norm{\tw_{t} - \tw_{t-1}}^2 + \left< \nabla \phit (\tw_{t-1}),\, \tw_{t} - \tw_{t-1} \right> \nonumber \\ 
    & = \phi (\tw_{t-1}) + \frac{1}{2 (\gamma_t  - \beta)} \norm{ \nabla \phi (\tw_{t-1}) - \nabla \phit (\tw_{t-1}) }^2 + \frac{\gamma_t}{2} \norm{\tw_{t} - \tw_{t-1}}^2 \nonumber \\ \label{e:minibatch-sgd-1}
    & \qquad\qquad\qquad\qquad\qquad\qquad\qquad\qquad
    + \left< \nabla \phit (\tw_{t-1}),\, \tw_{t} - \tw_{t-1} \right>
  \end{align}
  where we have used the Cauchy-Schwarz inequality in the second inequality, and the inequality $x y \le \frac{x^2}{2 \alpha} + \frac{\alpha y^2}{2}$ in the third inequality. 

  Now, since $\tw_{t}$ is the minimizer of the $\gamma_t$-strongly convex function 
\begin{align*}
\frac{\gamma_t}{2} \norm{\w - \tw_{t-1}}^2 + \left< \nabla \phit (\tw_{t-1}),\, \w - \tw_{t-1} \right>
\end{align*}
in $\Omega$, we have according to Lemma~\ref{lem:exact-distance-to-any} (replacing the local objective with its linear approximation) that 
  \begin{gather*}
    \frac{\gamma_t}{2} \norm{\w_* - \tw_{t-1}}^2 + \left< \nabla \phit (\tw_{t-1}),\, \w_* - \tw_{t-1} \right>  \\
    \ge \frac{\gamma_t}{2} \norm{\tw_{t} - \tw_{t-1}}^2 + \left< \nabla \phit (\tw_{t-1}),\, \tw_{t} - \tw_{t-1} \right> + \frac{\gamma_t}{2} \norm{\w_* - \tw_{t}}^2.
  \end{gather*}

  Substituting this into~\eqref{e:minibatch-sgd-1} gives
  \begin{align*}
    \phi (\tw_{t}) 
    & \le \phi (\tw_{t-1}) + \frac{1}{2 (\gamma_t - \beta)} \norm{ \nabla \phi (\tw_{t-1}) - \nabla \phit (\tw_{t-1}) }^2 + \frac{\gamma_t}{2} \norm{\w_* - \tw_{t-1}}^2 \\ 
    & \qquad\quad  + \left< \nabla \phit (\tw_{t-1}),\, \w_* - \tw_{t-1} \right> - \frac{\gamma_t}{2} \norm{\w_* - \tw_{t}}^2.
  \end{align*}
  Taking expectation of this inequality over the random sampling of $I_t$ further leads to
  \begin{align}
    \bbE_{I_t} \left[ \phi (\tw_{t})  \right]
    & \le \phi (\tw_{t-1}) + \frac{1}{2 (\gamma_t - \beta)} \bbE_{I_t} \norm{ \nabla \phi (\tw_{t-1}) - \nabla \phit (\tw_{t-1}) }^2 + \frac{\gamma_t}{2} \norm{\w_* - \tw_{t-1}}^2 \nonumber \\
    & \qquad\qquad\qquad\qquad\qquad\qquad
    + \left< \nabla \phi (\tw_{t-1}),\, \w_* - \tw_{t-1} \right> - \frac{\gamma_t}{2} \bbE_{I_t} \norm{\w_* - \tw_{t}}^2 \nonumber \\ \label{e:minibatch-sgd-2}
    & \le \phi (\w_*) + \frac{1}{2 (\gamma_t - \beta)} \bbE_{I_t} \norm{ \nabla \phi (\tw_{t-1}) - \nabla \phit (\tw_{t-1}) }^2 \nonumber \\
    & \qquad\qquad\qquad\qquad\qquad\qquad
    + \frac{\gamma_t-\lambda}{2} \norm{\w_* - \tw_{t-1}}^2 - \frac{\gamma_t}{2} \bbE_{I_t} \norm{\w_* - \tw_{t}}^2
  \end{align}
  where in the second inequality we have used the fact that 
  \begin{align*}
    \phi (\w_*) \ge \phi (\tw_{t-1}) + \left< \nabla \phi (\tw_{t-1}),\, \w_* - \tw_{t-1} \right> + \frac{\lambda}{2} \norm{\w_* - \tw_{t-1}}^2 
  \end{align*}
  due to the convexity of $\phi(\w)$. 
  
On the other hand, let $I_t=\{\xi_1,\dots,\xi_b\}$,  we have 
  \begin{align*}
    & \bbE_{I_t} \norm{ \nabla \phi (\w) - \nabla \phit (\w) }^2 \\
    =\ & \bbE_{I_t} \norm{ \nabla \phi (\w) - \frac{1}{b} \sum_{i=1}^b \nabla \ell (\w,\xi_i) }^2 \\
    =\ & \bbE_{I_t} \norm{  \frac{1}{b} \sum_{i=1}^b \left(\nabla \phi (\w) - \nabla \ell (\w,\xi_i) \right) }^2 \\
    =\ & \frac{1}{b^2} \sum_{i=1}^b \bbE_{\xi_i} \norm{ \nabla \phi (\w) - \nabla \ell (\w,\xi_i) }^2 + \frac{1}{b^2} \sum_{i \neq j} \bbE_{I_t} \left< \nabla \phi (\w) - \nabla \ell (\w,\xi_i) ,\, \nabla \phi (\w) - \nabla \ell (\w,\xi_j) \right> \\
    =\ & \frac{1}{b} \cdot \bbE_{\xi} \norm{ \nabla \phi (\w) - \nabla \ell (\w, \xi) }^2 \\
    \le\ & \frac{4L^2}{b}
  \end{align*}
where we used the fact that the samples are i.i.d. in the fourth equality, and that $\norm{\nabla \phi (\w)},\ \norm{\nabla \ell (\w, \xi)} \le L$ in the last inequality. 
  Continuing from~\eqref{e:minibatch-sgd-2} yields the desired result.
\end{proof}

Comparing this result to~\eqref{e:exact-minibatch-weakly-convex-distance-to-optimum} and~\eqref{e:exact-minibatch-strongly-convex-distance-to-optimum}, we observe that minibatch SGD has  a similar recursion to that exact minibatch-prox, except the appearance of $\beta$ in the denominator of the ``stability'' term. We now show that this difference leads to significant difference in convergence rate. 

Let $\ell(\w,\xi)$ be weakly convex ($\lambda=0$), and $\gamma_t=\gamma$ for all $t\ge 1$. Summing~\eqref{e:minibatch-sgd-time-t}  over $t=1,\dots,T$ gives
\begin{align*}
  \sum_{t=1}^T \bbE \left[ \phi (\tw_{t}) - \phi (\w_*) \right]
  \le 
  \frac{2 L^2 T}{(\gamma - \beta) b} + \frac{\gamma}{2} \norm{\w_* - \tw_{0}}^2.
\end{align*}
Minimizing the RHS over $\gamma$ gives
\begin{align*}
  \gamma = \beta + \sqrt{\frac{4 T}{b}} \cdot \frac{L}{\norm{\w_* - \tw_{0}}},
\end{align*}
which leads to 
\begin{align*}
  \frac{1}{T} \sum_{t=1}^T \bbE \left[ \phi (\tw_{t}) - \phi (\w_*) \right] \le
  \frac{2 L \norm{\w_* - \tw_{0}}}{\sqrt{bT}} + \frac{\beta \norm{\w_* - \tw_{0}}^2}{2 T}.
\end{align*}
So we obtain the familiar $\calO \left( \frac{1}{\sqrt{bT}} + \frac{1}{T} \right)$ rate for minibatch SGD. 

\section{Proof of Theorem~\ref{thm:main-mbdsvrg}}

\begin{proof}
  On the one hand, as we choose $\gamma$ as Theorem~\ref{thm:inexact-weakly-convex-sufficient} suggested, we just need to verify that the inexactness conditions in Theorem \ref{thm:inexact-weakly-convex-sufficient} is satisfied, i.e., for $t=1,\dots,T$, we require (recall that $\w_t^* = \argmin_{\w}\; \tilde f_t (\w)$) 
  \begin{align*}
    \tilde f_t(\w_t) - \tilde f_t (\w_t^*) & \leq \frac{1}{10^4} \cdot \min \left( \left( \frac{T}{bm} \right)^{1/2} , \left( \frac{T}{bm} \right)^{3/2} \right) \cdot \frac{LB}{t^{3}} .
  \end{align*}
  On the other hand, we can bound the initial suboptimality of $\tilde f_t (\w)$ when initializing from $\w_{t-1}$. This is because, by the optimality of $\w_t^*$, we have $\norm{\w_t^* - \w_{t-1}} =  \norm{\frac{1}{\gamma} \nabla \phit(\w_t^*)} \le L/\gamma$, and 
  \begin{align} % 
    \tilde f_t (\w_{t-1}) - \tilde f_t (\w_t^*) & = 0 + \phit(\w_{t-1}) - \frac{\gamma}{2} \norm{\w_t^* - \w_{t-1}}^2 - \phit(\w_t^*) \nonumber \\  \label{e:mbdsvrg-init}
&  \le  \phit(\w_{t-1}) - \phit(\w_t^*) \le L \norm{\w_t^* - \w_{t-1}} 
    \le L^2 / \gamma .
  \end{align}
  Combining the above two inequalities, the initial versus final error for the $K$ DSVRG iterations is bounded by 
  \begin{align*}
    &10^{4} \cdot \max \left( \left( \frac{bm}{T} \right)^{1/2} ,\, \left( \frac{bm}{T} \right)^{3/2} \right) \cdot t^3 \cdot \frac{L}{B \gamma} \\
    =\ & 10^{4} \cdot \max \left( \left( \frac{bm}{T} \right)^{1/2} ,\, \left( \frac{bm}{T} \right)^{3/2} \right) \cdot T^3 \cdot \frac{L}{B} \cdot \frac{bmB}{\sqrt{8 n (\varepsilon)} L} \\
    =\  & \calO \left( \max \left( \frac{n(\varepsilon)^{2}}{b m},\, bm \cdot n(\varepsilon) \right) \right) \\
    =\ &  \calO \left( n^2 (\varepsilon)  \right)
  \end{align*}
  where we have used the definition of $\gamma$ and $T = \frac{n(\varepsilon)}{b m}$ in the first and second step respectively.

  By the iteration complexity results for sampling without-replacement DSVRG~\citep[Theorem~4]{shamir2016without}, we have the desired suboptimality in $\tilde f_t (\w)$ using
  \begin{align} \label{e:mbdsvrg-K}
    K = \calO \left( \log n(\varepsilon) \right)
  \end{align}
iterations, %% where we have plugged in the value of $\gamma$ in the second step.
  as long as the batch size $b/p_i$ is larger than the problem condition number.

Now, the condition number of $\tilde f (\w)$ is 
  \[
  \frac{\beta + \gamma}{\gamma} = \calO \left( \frac{\beta b m B}{\sqrt{n(\varepsilon)} L} \right).
  \]
  Equating this to the batch size $b / p_i$ yields the $p_i$ specified in the theorem. It is also easy to check that $\frac{K}{\gamma}=\calO (bm)$, \ie, the total number of stochastic updates is less than the total number of samples, as required by~\citet[Theorem~4]{shamir2016without}.
  
  \textbf{Communication:} the total rounds of communication required by Algorithm~\ref{alg:mbdsvrg} is
  \[
  KT = \calO \left( \frac{n(\varepsilon)}{mb} \log n(\varepsilon) \right) .
  \]

  \textbf{Computation:} For each communication round, each machine need to compute the local full gradient, which can be done in parallel, and then one of the machines perform $b/p_i$ steps of stochastic update. So the computation cost is
  \[
  KT \left( b + \frac{b}{p_i} \right) = \calO \left( \frac{n(\varepsilon)}{m} \log n(\varepsilon) \right). 
  \]
  \textbf{Memory:} It is straightforward to see each machine only need to maintain $b$ samples. 
\end{proof}
\section{Communication-efficient distributed minibatch-prox with DANE}
\label{sec:stochastic-dane}

\begin{algorithm}[!t]
  \caption{MP-DANE for distributed stochastic convex optimization.}
  \label{alg:mp-ppa-dane}
  \renewcommand{\algorithmicrequire}{\textbf{Input:}}
  \renewcommand{\algorithmicensure}{\textbf{Output:}}
  \begin{algorithmic}
    \STATE Initialize $\w_0$.
    \FOR{$t=1,2,\dots,T$}
    \STATE Each machine $i$ draws a minibatch $I_t^{(i)}$ of $b$ samples from the underlying data distribution.
    \STATE Initialize $\y_0 \leftarrow \w_{t-1},\quad \x_0 \leftarrow \w_{t-1}$.
    \FOR{$r=1,2,\dots,R$}
    \STATE Initialize $\z_0 \leftarrow \y_{r-1}$, $\alpha_0 = \sqrt{\gamma/(\gamma + \kappa)}$.
    \FOR{$k=1,2,\dots,K$}
    \STATE 1. All machines perform one round of communication to compute the average gradient 
    \begin{align*} \nabla \phit (\z_{k-1}) \leftarrow \frac{1}{m} \sum_{i=1}^m \nabla \phii (\z_{k-1}). 
    \end{align*}
    \vspace{-2ex}
  \item 2. Each machine $i$ approximately solves the local objective to $\theta$-accuracy: 
    \begin{gather} 
      \text{apply prox-SVRG to find} \; \z_k^{(i)} \quad \text{s.t.}\quad  
      \norm{ \z_k^{(i)} - \z_k^{(i)*} } \leq \theta \norm{ \z_{k-1} - \z_k^{(i)*} }  \nonumber \\  
      \text{where}\; \z_k^{(i)^*} = \argmin_{\z \in \Omega} \; 
      \phi_{I_t^{(i)}} (\z) + \left< \nabla \phit (\z_{k-1}) - \nabla \phii (\z_{k-1}),\, \z \right> + \frac{\gamma}{2} \norm{\z - \w_{t-1}}^2  \nonumber \\  \label{e:svrg-local_solver}  
      + \frac{\kappa}{2} \norm{\z - \y_{r-1}}^2.
    \end{gather} 
    \vspace{-2ex}
    \STATE 3. All machines reach consensus by averaging local updates through another round of communication:
    \begin{align}
      \z_k \leftarrow \frac{1}{m} \sum_{i=1}^m \z_k^{(i)}.
    \end{align}
    \vspace{-2ex}
    \ENDFOR
    \STATE Update $\x_r \leftarrow \z_K$. 
    \STATE Compute $\alpha_r \in (0,1)$ such that $\alpha_r^2 = (1 - \alpha_r)\alpha_{r-1}^2 + \gamma \alpha_k/(\gamma+\kappa)$, and compute 
    % \STATE  Compute
    \begin{align}
      \y_r = \x_r + \left( \frac{ \alpha_{r-1}(1-\alpha_{r-1}) }{ \alpha_r + \alpha_{r-1}^2 } \right)(\x_r - \x_{r-1}). 
      \label{e:y_extrapolation}
    \end{align}
    \vspace{-2ex}
    \ENDFOR
    \STATE Update $\w_t \leftarrow \x_r$. 
    \ENDFOR
    \ENSURE $\w_T$ is the approximate solution.
  \end{algorithmic}
\end{algorithm}

As discussed in Section \ref{sec:mbdsvrg}, it is also possible to use other efficient distributed optimization solver for minibatch-prox. 
Here we present a novel method that use the distributed optimization algorithm DANE~\citep{Shamir_14a} and its accelerated variant AIDE~\citep{reddi2016aide} for solving~\eqref{e:bm-minibatch-problem}, which define better local objectives than EMSO and take into consideration the similarity between local objectives. 

We detail our algorithm, named MP-DANE, in Algorithm~\ref{alg:mp-ppa-dane}. The algorithm consists of three nested loops, where $t$, $r$ and $k$ are iteration counters for minibatch-prox (the outer \textbf{for}-loop), AIDE (the intermediate \textbf{for}-loop) and DANE (the inner \textbf{for}-loop) respectively.
Compared to EMSO, DANE adds a gradient correction term to~\eqref{e:distributed-emso-local-objective} which can be compute efficiently with one round of communication. 
On top of that, AIDE uses the idea of universal catalyst~\citep{lin2015universal} and adds an extra quadratic term to improve the strong-convexity of the objective for faster convergence, \ie, in order to solve~\eqref{e:bm-minibatch-problem}, AIDE solves multiple instances of the ``augmented large minibatch'' problems of the form 
\begin{align} \label{e:augmented-large-minibatch}
  \min_{\w \in \Omega}\; \bar f_{t,r} (\w) := \phi_{I_t} (\w) + \frac{\gamma}{2} \norm{ \w - \w_{t-1}}^2 + \frac{\kappa}{2} \norm{\w - \y_{r-1}}^2
\end{align}
with carefully chosen extrapolation points $\y_{r-1}$. 
At each DANE iteration, we perform two rounds of communication, one for averaging the local gradients, and one for averaging the local updates, and the amount of data we communicate per round has the same size of the predictor. 

To sum up, in Algorithm~\ref{alg:mp-ppa-dane}, we have introduced two levels of inexactness.  First, we only approximately solve the ``large minibatch'' subproblem~\eqref{e:bm-minibatch-problem} in each outer loop; results from the previous section guarantee the convergence of this approach. Second, we only approximately solve the local subproblems~\eqref{e:svrg-local_solver} to sufficient accuracy in each inner loop; the analysis of ``inexact DANE'' (for the non-stochastic setting) provides guarantee for this approach~\citep{reddi2016aide}, and  enables us to use state-of-the-art SGD methods (\eg, SVRG~\citealp{johnson2013accelerating,xiao2014proximal}) for solving local subproblems. 
Overall, we obtain a convergent algorithm for distributed stochastic convex optimization. 

%% \weiran{Is Algorithm~\ref{alg:mp-ppa-dane} tailored to least squares problems? I remember we need to choose $\mu$ in DANE for general convex problems. If the algorithm is tailored to least squares, we shall mention it.}
%% \weiran{You use SVRG in this section, but SAGA in the experiments?} 

We now present detailed analysis for the computation/communication complexity of Algorithm~\ref{alg:mp-ppa-dane} for stochastic quadratic problems, and compare it with related methods in the literature. 

\subsection{Efficiency of MP-DANE}

We present the main results of this section (full analysis is deferred to Appendix~\ref{sec:full-dane-analysis}), which show that with careful choices of the minibatch size and the desired accuracy in each level of approximate solution, MP-DANE achieves both communication and computation efficiency with the optimal sample complexity. %% Appendix~\ref{sec:full-dane-analysis}, we establish auxiliary results that characterize the iteration complexity of solving the local problem~\eqref{e:svrg-local_solver} by prox-SVRG~\citep{xiao2014proximal}, the large minibatch problem~\eqref{e:bm-minibatch-problem} by DANE~\citep{Shamir_14a} and AIDE~\citep{reddi2016aide}. 
Interestingly, the choices of parameters differ in two regimes which are separated by an ``optimal'' minibatch size (also denoted as $b_{\operatorname{mp-dane}}$ in the main text)
\begin{align*}
b^* = \frac{n(\varepsilon) L^2 }{32  m^2 \beta^2 B^2 \log(md)}.
\end{align*}

\begin{thm}[Efficiency of MP-DANE for $b \leq b^*$]
  \label{thm:main-stochastic-dane-small-b}
  Set the parameters in Algorithm~\ref{alg:mp-ppa-dane} as follows:
  \begin{align*}
    \text{(outer loop)} & \qquad  b \leq b^* = \frac{n(\varepsilon) L^2 }{32  m^2 \beta^2 B^2 \log(md)}, \quad T = \frac{n(\varepsilon)}{bm}, \quad \gamma = \frac{\sqrt{8 n(\varepsilon)}  L}{b m B}, \\
    \text{(intermediate loop)} & \qquad \kappa = 0, \quad  R = 1, \\
    \text{(inner loop)} & \qquad  \theta = \frac{1}{6}, \quad K =  \calO \left( \log n(\varepsilon) \right) .
  \end{align*}
  Then we have
  $
  \bbE \left[ \phi \left( \frac{1}{T} \sum_{t=1}^T \w_t \right) - \phi(\w_*) \right] \leq \frac{\sqrt{40} B L}{\sqrt{n(\varepsilon)}}  =  \calO \left( \varepsilon \right).
  $
  
  Moreover, Algorithm \ref{alg:mp-ppa-dane} can be implemented with 
  $
\tilde{\calO} \left( \frac{n(\varepsilon)}{bm} \right)
  $
  rounds of communication, and each machine performs 
  $
\tilde{\calO} \left( \frac{n(\varepsilon)}{m} \right)
  $
  vector operations in total, % and each machine requires $\calO (b)$ memory, 
  where the notation $\tilde{\calO} (\cdot)$ hides poly-logarithmic dependences on $n(\varepsilon)$. 

  When we choose $b = b^*$, Algorithm 1 can be implemented with 
  $
\tilde{\calO} \left( \frac{m \beta^2 B^2}{L^2} \right)
  $
  rounds of communication, 
  $
  \tilde{\calO} \left( \frac{n(\varepsilon)}{bm} \right)
  $ 
  vector operations, and $\calO \left( \frac{n(\varepsilon) L^2}{m^2 \beta^2 B^2} \right)$ memory for each machine.
\end{thm}

We comment on the choice of parameters. For sample efficiency, we fix the sample size $n (\varepsilon)$ and number of machines $m$, and so we can tradeoff the local minibatch size $b$ and the total number of outer iterations $T$, maintaining $b T = \frac{n (\varepsilon)}{m}$. For any $b$, the regularization parameters in the ``large minibatch'' problem is set to $\gamma = \sqrt{\frac{8T}{bm}} \cdot \frac{L}{B} = \frac{\sqrt{8 n (\varepsilon)} L}{bmB}$ according to Theorem~\ref{thm:inexact-weakly-convex-sufficient}. 
When $b \le b^*$, we note that~\eqref{e:b-gamma-ineq} can be satisfied with $\kappa=0$ and there is no need for acceleration by AIDE ($R=1$). Then the values of $\theta$ and $K$ follow from Lemma~\ref{lem:inexact-dane}.

\begin{rmk}
  The above theorem suggests that in the regime of $b\le b^*$, we only need to have logarithmic number of DANE iterations for solving each ``large minibatch'' problem, and logarithmic number of passes over the local data during each DANE iteration.  We present experimental results validating our theory in Appendix~\ref{sec:expts}.
\end{rmk}

The next theorem shows that when we use a large minibatch size $b$ in Algorithm~\ref{alg:mp-ppa-dane}, we can still satisfy the condition~\eqref{e:b-gamma-ineq} by adding extra regularization ($\kappa >0$), and then apply accelerated DANE. 

\begin{thm}[Efficiency of MP-DANE for $b \geq b^*$]
  \label{thm:main-stochastic-dane-large-b}
  Set the parameters in Algorithm~\ref{alg:mp-ppa-dane} as follows:
  \begin{align*}
    \text{(outer loop)} & \qquad  b \geq b^* = \frac{n(\varepsilon) L^2 }{32  m^2 \beta^2 B^2 \log(md)}, \quad T = \frac{n(\varepsilon)}{bm}, \quad \gamma = \frac{\sqrt{8 n(\varepsilon)}  L}{b m B}, \\
    \text{(intermediate loop)} & \qquad \kappa = 16 \beta \sqrt{\frac{\log(dm)}{b}}  - \gamma, \quad R = \calO \left( \frac{b^{1/4} m^{1/2} \cdot \beta^{1/2} B^{1/2}}{ n(\varepsilon)^{1/4} \cdot L^{1/2}} \log n(\varepsilon) \right), \\
    \text{(inner loop)} & \qquad  \theta = \frac{1}{6}, \quad K = \calO \left( \log n(\varepsilon) \right) .
  \end{align*}
Then we have $
  \bbE \left[ \phi \left( \frac{1}{T} \sum_{t=1}^T \w_t \right) - \phi(\w_*) \right] \leq \frac{\sqrt{40} B L}{\sqrt{n(\varepsilon)}}   =  \calO \left( \varepsilon \right).
  $

  Moreover, Algorithm \ref{alg:mp-ppa-dane} can be implemented with 
  $
\tilde{\calO} \left( \frac{ n(\varepsilon)^{3/4} \cdot \beta^{1/2} B^{1/2} }{ b^{3/4} m^{1/2} \cdot L^{1/2}} \right)
  $
  rounds of communication, and each machine performs 
  $
\tilde{\calO} \left( \frac{ b^{1/4} n(\varepsilon)^{3/4} \cdot \beta^{1/2} B^{1/2} }{  m^{1/2} \cdot L^{1/2}} \right)
  $
  vector operations in total, % and each machine requires $\calO (b)$ memory, 
  where the notation $\tilde{\calO} (\cdot)$ hides poly-logarithmic dependences on $n(\varepsilon)$. 
\end{thm}

\subsection{Two regimes of multiple resource tradeoffs}

From the above analysis, we summarized in Table~\ref{table:trade-offs} the resources required by MP-DANE. 
We observe two interesting regimes, separated by the minibatch size $b^* \asymp n(\varepsilon)/(m^2 B^2)$, that present different tradeoffs between communication, computation and memory. 

\begin{table*}[t]
  \centering
  \begin{tabular}{|c||c|c|c|c|}
    \hline
    & Samples & Communication & Computation & Memory  \\\hline\hline
    $1 \leq b \leq
    b^*$  &  $n(\varepsilon)$    &  $n(\varepsilon)/mb$   &   $n(\varepsilon)/m$ & $b$ \\
    $b = b^*$ &  $n(\varepsilon)$    &  $B^2 m$   &   $n(\varepsilon)/m$ & $n(\varepsilon)/(m^2 B^2)$  \\
    $b^* < b \leq b_{\max}$  &$n(\varepsilon)$ &  $B^{1/2} n(\varepsilon)^{3/4}/(m^{1/2} b^{3/4})$    &  $B^{1/2} n(\varepsilon)^{3/4} b^{1/4}/m^{1/2} $   & $b$ \\
    \hline
  \end{tabular}
  \caption{Summary of resources required by MP-DANE for
    distributed stochastic convex optimization, in units of vector 
    operations/communications/memory per machine, ignoring constants and $\log$-factors. Here $b^* \asymp n(\varepsilon)/(m^2 B^2)$, and $b_{\max} = n(\varepsilon)/m$.}
  \label{table:trade-offs}
\end{table*}

\begin{itemize}
\item When $1 \leq b \leq b^*$, the computation complexity remains $\tilde \calO \left( n(\varepsilon) / {m} \right)$ which is independent of $b$. This means we always achieve \emph{near-linear speedup} in this regime. Moreover, there is a tradeoff between communication and memory: the communication complexity decreases, while the memory cost increases as the minibatch size $b$ increases, both at the linear rate. Thus in this regime, we can trade communication for memory without affecting computation.
\item When $b^* < b \leq b_{\max}$, the computation starts to increase with $b$ at the rate % $\tilde \calO \left( \frac{B^{1/2} n(\varepsilon)^{3/4} b^{1/4}}{m^{1/2}} \right) $ 
$b^{1/4}$ which is slower than linear, while the communication cost continues to decrease at the rate $b^{3/4}$ which is also 
% $\tilde \calO \left( \frac{B^{1/2} n(\varepsilon)^{3/4}}{m^{1/2} b^{3/4}} \right)$%
slower than linear. Thus in this regime, we can trade communication for computation and memory.
\end{itemize}

\subsection{Analysis of MP-DANE}
\label{sec:full-dane-analysis}

In order to fully analyze Algorithm~\ref{alg:mp-ppa-dane}, we need several auxiliary lemmas that characterize the iteration complexity of solving the local problem~\eqref{e:svrg-local_solver} by prox-SVRG~\citep{xiao2014proximal}, the large minibatch problem~\eqref{e:bm-minibatch-problem} by DANE~\citep{Shamir_14a} and AIDE~\citep{reddi2016aide}.

\subsubsection{Some auxiliary lemmas}
\label{sec:distributed-auxiliary}

First, we apply prox-SVRG to the local problem~\eqref{e:svrg-local_solver}, pushing all terms but $\phi_{I_t^{(i)}} (\z)$ in to the proximal operator. %%; this does not complicate the computation since the proximal step simply computes the Euclidean projection onto $\Omega$. 
The benefit of this approach (as opposed to using plain SVRG~\citealp{johnson2013accelerating}) is that the smoothness parameter that determines the iteration complexity is simply $\beta$, same results hold when applying prox-SAGA \citep{defazio2014saga} as well. For sampling without replacement
SVRG, the current analysis works only for plain SVRG, so we quote the results from \citep{shamir2016without}.

\begin{lem}[Iteration complexity of SVRG for~\eqref{e:svrg-local_solver}]
  \label{lem:svrg-complexity}
  For any target accuracy $\theta > 0$, with initialization $\z_{k-1}$, prox-SVRG outputs $\z_k^{(i)}$ such that 
  $\norm{ \z_k^{(i)} - \z_k^{(i)*} } \leq \theta \norm{ \z_{k-1} - \z_k^{(i)*} }$
  after 
  \begin{align*}
    \calO \left( \left( b + \frac{\beta}{\gamma + \kappa} \right) \cdot \log \frac{(\beta+\gamma+\kappa)}{(\gamma + \kappa)\theta^2} \right)
  \end{align*}
  vector operations, and sampling without replacement SVRG outputs $\z_k^{(i)}$ such that 
  $\norm{ \z_k^{(i)} - \z_k^{(i)*} } \leq \theta \norm{ \z_{k-1} - \z_k^{(i)*} }$
  after 
  \begin{align*}
    \calO \left( \left( b + \frac{\beta + \kappa}{\gamma + \kappa} \right) \cdot \log \frac{(\beta+\gamma+\kappa)}{(\gamma + \kappa)\theta^2} \right)
  \end{align*}
  vector operations.
\end{lem} 

\begin{proof}
  Observe that the objective~\eqref{e:svrg-local_solver} by $f_k^{(i)} (\z)$, which is an quadratic function of $\z$ with the Hessian matrix $H_i = \nabla^2 \phii(\z) + (\gamma + \kappa) \I \succeq (\gamma + \kappa) \I$. As a result, the suboptimality of $\z_k^{(i)}$ is 
  \begin{align*}
    \epsilon_{\text{final}}= f_k^{(i)} (\z_k^{(i)}) - f_k^{(i)} (\z_k^{(i)*}) = \frac{1}{2} \left( \z_k^{(i)} - \z_k^{(i)*} \right)^\top H_i \left( \z_k^{(i)} - \z_k^{(i)*} \right) \ge \frac{\gamma + \kappa}{2} \norm{\z_k^{(i)} - \z_k^{(i)*} }^2. 
  \end{align*}
  To satisfy the requirement of $\norm{ \z_k^{(i)} - \z_k^{(i)*} } \leq \theta \norm{ \z_{k-1} - \z_k^{(i)*} }$, we require 
  \begin{align*}
    \epsilon_{\text{final}} \le \frac{(\gamma + \kappa) \theta^2}{2} \norm{\z_{k-1} - \z_k^{(i)*} }^2.
  \end{align*}

  On the other hand, when initializing from $z_{k-1}$, the initial suboptimality is 
  \begin{align*}
    \epsilon_{\text{init}} = f_k^{(i)} (\z_{k-1}) - f_k^{(i)} (\z_k^{(i)*}) \le \frac{\sigma_{\max} (H_i)}{2} \norm{\z_{k-1} - \z_k^{(i)*} }^2 \le \frac{\beta+\gamma+\kappa}{2} \norm{\z_{k-1} - \z_k^{(i)*} }^2. 
  \end{align*}

  Therefore, it suffices to have 
  \begin{align*}
    \frac{\epsilon_{\text{init}}}{\epsilon_{\text{final}}} = \frac{(\beta+\gamma+\kappa)}{(\gamma + \kappa) \theta^2}. 
  \end{align*}

  Noting that $\phii(\z)$ is the sum of $b$ components, and each component is $\beta$-smooth while the overall function $f_k^{(i)}$ is $(\gamma+\kappa)$-strongly convex, the lemma follows directly from the convergence guarantee of prox-SVRG~\citep[Corollary 1]{xiao2014proximal}, and sampling without replacement SVRG~\citep[Theorem 4]{shamir2016without}.
\end{proof}

Next, we state the convergence rates of ``inexact DANE'' and AIDE, which can be easily derived from~\citet{reddi2016aide}. 
At the outer loop $t$ and intermediate loop $r$,  let $\x_r^* = \argmin_{\w}\; \bar f_{t,r} (\w)$ be the exact minimizer of the ``augmented large minibatch'' problem~\eqref{e:augmented-large-minibatch}, which is approximately solved by the inner DANE iterations.

\begin{lem}[Iteration Complexity of inexact DANE] \label{lem:inexact-dane}
  Let $\theta = \frac{1}{6}$, and assume that 
  \begin{align}
    \label{e:b-gamma-ineq}
    b (\gamma + \kappa)^2 \geq 256 \beta^2 \log(dm / \delta).
  \end{align}
  By initializing from $\y_{r-1}$, and setting the number of inner iterations in Algorithm \ref{alg:mp-ppa-dane} to be 
  \begin{align*}
    % K = \ceil{ \frac{ \log (1/\eta) }{ \log (4/3)} },
    K = \ceil[\Big]{ \frac{1}{2} \log_{4/3} \frac{(\beta+\gamma+\kappa)}{ (\gamma+\kappa) \eta} },
  \end{align*}
  we have with probability $1-\delta$ over the sample set $I_t$ that 
  % \[
  % \norm{ \z_K - \x_r^* } \leq \eta \norm{ \y_{r-1} - \x_r^* }.
  % \]
  \begin{align*}
    % \bar f_{t,r} (\z_K) - \bar f_{t,r} (\x_r^*) \leq 
    \bar f_{t,r} (\x_r) - \bar f_{t,r} (\x_r^*) \leq 
    \eta \left( \bar f_{t,r} (\y_{r-1}) - \bar f_{t,r} (\x_r^*)\right) .
  \end{align*}
\end{lem}
\begin{proof}
  Denote by $H_i = \nabla^2 \phii(\z) + (\gamma + \kappa) \I$ the Hessian matrix of the local objective~\eqref{e:svrg-local_solver} for machine $i$. 
  Let $H = \frac{1}{m} \sum_{i=1}^m H_i$ be the Hessian matrix of the global objective~\eqref{e:augmented-large-minibatch}, and $\tilde H^{-1} = \frac{1}{m} \sum_{i=1}^m H_i^{-1}$. 
  As our objective is quadratic, $H_i,H,\tilde H^{-1}$ remain unchanged during the inner iterations. By~\citet[Theorem~1]{reddi2016aide}, we have
  \begin{align} \label{e:inexact-dane-1}
    \norm{ \z_k - \x_r^* } \leq \left( \norm{\tilde H^{-1} H - \I } + \frac{\theta}{m} \sum_{i=1}^m \norm{ H_i^{-1} H}   \right) \norm{ \z_{k-1} - \x_r^* }.
  \end{align}

  Since $\nabla^2 \ell (\w,\xi) \le \beta$, by~\citet[Lemma~2]{Shamir_14a}, we have with probability at least $1 - \delta$ over the sample set $I_t$ that 
  \begin{align*}
    \norm{ H_i - H } \leq \sqrt{ \frac{32 \beta^2 \log(dm/\delta) }{b} } =: \rho,\qquad\qquad i=1,\dots,m. 
  \end{align*}
  On the other hand, we have $H_i \succeq (\gamma+\kappa) \I$ and 
  \begin{align*}
    \frac{4 \rho^2}{(\gamma+\kappa)^2} =  \frac{128 \beta^2 \log(dm/\delta) }{b (\gamma+\kappa)^2} \leq \frac{1}{2}
  \end{align*}
  by our assumption~\eqref{e:b-gamma-ineq}. By~\citet[Lemma~1]{Shamir_14a}, we have
  \begin{align} \label{e:inexact-dane-2}
    \norm{\tilde H^{-1} H - \I } \leq \frac{1}{2}.
  \end{align}
  Moreover, we have 
  \begin{align}
    \frac{\theta}{m} \sum_{i=1}^m \norm{ H_i^{-1} H}
    & \leq \frac{\theta}{m} \sum_{i=1}^m (1 + \norm{ H_i^{-1} H - \I } ) \nonumber \\
    & \leq \frac{\theta}{m} \sum_{i=1}^m (1 + \norm{H_i^{-1}} \norm{  H - H_i^{-1} } ) \nonumber \\
    & \leq \frac{\theta}{m} \sum_{i=1}^m \left( 1 + \frac{\rho}{\gamma + \kappa} \right) \nonumber \\
    & \leq \frac{\theta}{m} \sum_{i=1}^m \left( 1 + \frac{1}{2\sqrt{2}} \right) \nonumber \\ \label{e:inexact-dane-3}
    & \leq \frac{3 \theta}{2} \leq \frac{1}{4} . 
  \end{align}
  Plugging~\eqref{e:inexact-dane-2} and~\eqref{e:inexact-dane-3} into~\eqref{e:inexact-dane-1} yields 
  \[
  \norm{ \z_k - \x_r^* } \leq \frac{3}{4} \norm{ \z_{k-1} - \x_r^* },
  \] 
  and thus $\norm{ \z_K - \x_r^* } \leq (3/4)^K \norm{ \y_{r-1} - \x_r^* } $. %% Then the lemma follows.
  To guarantee the suboptimality in the objective $\bar f_{t,r} (\w)$, we note that 
  \begin{align*}
    \bar f_{t,r} (\z_K) - \bar f_{t,r} (\x_r^*) & = \frac{1}{2} (\z_K - \x_r^*)^\top H (\z_K - \x_r^*) \le \frac{\beta+\gamma+\kappa}{2} \norm{\z_K - \x_r^*}^2  \\
    & \le \left(\frac{3}{4}\right)^{2K} \frac{\beta+\gamma+\kappa}{2} \norm{\y_{r-1} - \x_r^*}^2 \\
    & \le \left(\frac{3}{4}\right)^{2K} \frac{\beta+\gamma+\kappa}{\gamma+\kappa} \left( \bar f_{t,r} (\y_{r-1}) - \bar f_{t,r} (\x_r^*) \right)
  \end{align*}
  where we have used the fact that $f_{t,r}(\w)$ is $(\gamma+\kappa)$-strongly convex in the last inequality. 
  Setting $\left(\frac{3}{4}\right)^{2K} \frac{\beta+\gamma+\kappa}{\gamma+\kappa} = \eta$, and noting $\x_r = \z_K$, we obtain the desired iteration complexity. 
\end{proof}

At the outer iteration $t$ of Algorithm~\ref{alg:mp-ppa-dane}, we are trying to approximately minimize the objective~\eqref{e:bm-minibatch-problem} by iteratively (approximately) solving $R$ instances of the ``augmented'' problem~\eqref{e:augmented-large-minibatch}. 
% Note that $\y_{r-1}$ in~\eqref{e:augmented-large-minibatch} is an extrapolation point (as defined in~\eqref{e:y_extrapolation}) to achieve acceleration. 
Let $\w_t^*$ be the exact minimizer of the ``large minibatch'' subproblem~\eqref{e:bm-minibatch-problem}: 
\begin{align*}
  \w_t^* = \argmin_{\w}\; \tilde f_t (\w). 
\end{align*}
The following lemma characterizes the accelerated convergence rate. 

\begin{lem}[Acceleration by universal catalyst, Theorem~3.1 of~\citet{lin2015universal}] 
  \label{lem:catalyst-inexactness}
  Assume that for all $r\ge 1$, we have
  \[
  \bar f_{t,r}(\x_r) -  \bar f_{t,r}(\x_r^*) \leq \frac{2}{9} \left( 1 -  \frac{9}{10} \sqrt{\frac{\gamma}{\gamma + \kappa}} \right)^{R}  \cdot \left( \tilde f_t(\x_0) - \tilde f_t(\w_t^*) \right)  , 
  \]
  then 
  \[
  \tilde f_t(\x_R) - \tilde f_t(\w_t^*) \leq \frac{800(\gamma + \kappa)}{\gamma} \left( 1 -  \frac{9}{10} \sqrt{\frac{\gamma}{\gamma + \kappa}} \right)^{R+1}  \left( \tilde f_t(\x_0) - \tilde f_t(\w_t^*) \right).
  \]
\end{lem}

%% \subsection{Proofs of main results}
\subsubsection{Proof of Theorem~\ref{thm:main-stochastic-dane-small-b}}

\begin{proof}
  First of all, because $R=1$, our algorithm collapses into two nested loops.

  On the one hand, as we choose $\gamma$ as Theorem~\ref{thm:inexact-weakly-convex-sufficient} suggested, we just need to verify the inexactness conditions in Theorem \ref{thm:inexact-weakly-convex-sufficient} is satisfied, i.e., for $t=1,\dots,T$, we require (recall that $\w_t^* = \argmin_{\w}\; \tilde f_t (\w)$) 
  \begin{align*}
    \tilde f_t(\w_t) - \tilde f_t (\w_t^*) & \leq \frac{1}{10^4} \cdot \min \left( \left( \frac{T}{bm} \right)^{1/2} , \left( \frac{T}{bm} \right)^{3/2} \right) \cdot \frac{2LB}{t^{3}} .
  \end{align*}

  On the other hand, we can bound the initial suboptimality $\tilde f_t (\w)$ (cf. derivation for~\eqref{e:mbdsvrg-init}):
  \begin{align*}
    \tilde f_t (\tw_{t-1}) - \tilde f_t (\w_t^*) 
\le L^2 / \gamma.
  \end{align*}

%   Combining the above two inequalities, it suffices to set the initial versus final error for the $K$ inexact DANE iterations to be 
%   \begin{align*}
%     \frac{1}{\eta} & = 10^{4} \cdot \max \left( \left( \frac{bm}{T} \right)^{1/2} ,\, \left( \frac{bm}{T} \right)^{3/2} \right) \cdot t^3 \cdot \frac{ (\beta+\gamma)B}{L} \\
%     & \le  10^{4} \cdot \max \left( (bm)^{1/2} T^{5/2},\, (bm)^{3/2} T^{3/2} \right) \cdot \frac{ (\beta+\gamma)B}{L} \\
%     & = 10^{4} \cdot \max \left( \frac{n(\varepsilon)^{5/2}}{b^2 m^2},\, n(\varepsilon)^{3/2} \right) \cdot \frac{ (\beta+\gamma)B}{L}
%   \end{align*}
%   where we have substituted in $T = \frac{n(\varepsilon)}{b m}$ in the last step.

  Using Lemma~\ref{lem:inexact-dane}, we know as long as the inequality~\eqref{e:b-gamma-ineq} is satisfied, we have the desired suboptimality in $\tilde f_t (\w)$ using (cf. the derivation for~\eqref{e:mbdsvrg-K})
   \begin{align*}
     K & = \calO \left( \log n(\varepsilon)  \right)
   \end{align*}
  rounds of communication, where we have plugged in the value of $\gamma$ in the second step.

  It remains to verify the condition~\eqref{e:b-gamma-ineq}, by our choice of $\gamma$ and $b$, we have
  \begin{align} \label{e:distributed-small-b-and-gamma}
    b \gamma^2 = \frac{8 n(\varepsilon) L^2}{b m^2 B^2} \geq \frac{8 n(\varepsilon) L^2}{b^* m^2 B^2} = 256 \beta^2 \log(md),
  \end{align}
  as desired.

  Next we summarize the communication, computation, and memory efficiency.

  \textbf{Communication:} the total rounds of communication required by Algorithm~\ref{alg:mp-ppa-dane} is
  \[
  KRT = \calO \left( \frac{n(\varepsilon)}{mb} \log n(\varepsilon) \right) .
  \]

  \textbf{Computation:} For each communication round, we need to solve the local problem~\eqref{e:svrg-local_solver} using prox-SVRG. Now, in view of~\eqref{e:distributed-small-b-and-gamma}, we have $\beta = \calO ( \sqrt{b} \gamma)$. This implies that $\frac{\beta}{\gamma} = \calO (\sqrt{b})$ and thus by Lemma~\ref{lem:svrg-complexity}, the dominant term of the iteration complexity of prox-SVRG is
  \begin{align*}
    \calO \left( b \log \frac{\beta+\gamma}{\gamma} \right) = \calO \left( b \log  n(\varepsilon) \right).
  \end{align*}
  Multiplying this with the number of communication rounds yields the desired computation complexity. 

  \textbf{Memory:} It is straightforward to see each machine only need to maintain $b$ samples. 
\end{proof}

\subsubsection{Proof of Theorem~\ref{thm:main-stochastic-dane-large-b}}

\begin{proof}
  First, it is straightforward to verify the condition~\eqref{e:b-gamma-ineq}:
  \[
  b(\gamma + \kappa)^2 = 256 \beta^2 \log(dm). 
  \] 

Similarly to Theorem~\ref{thm:main-stochastic-dane-small-b}, we need the ratio between final versus initial error for the $R$ AIDE iterations to be 
  \begin{align*}
{\rm ratio} = \calO ( n(\varepsilon) ).
  \end{align*}
Equating this ratio to be $\frac{800(\gamma + \kappa)}{\gamma} \left( 1 -  \frac{9}{10} \sqrt{\frac{\gamma}{\gamma + \kappa}} \right)^{R+1}$, we have  
  \begin{align*}
    R & = \frac{10}{9} \sqrt{\frac{\gamma + \kappa}{\gamma}} \log \left( \frac{800 (\gamma+\kappa)}{\gamma} \cdot \frac{1}{ratio} \right) \\
& = \calO \left( \frac{b^{1/4} m^{1/2} \cdot \beta^{1/2} B^{1/2}}{ n(\varepsilon)^{1/4} \cdot L^{1/2}} \log n(\varepsilon) \right).
  \end{align*}

Now according to Lemma~\ref{lem:catalyst-inexactness}, the final suboptimality for $\bar f_{t,r} (\w)$ need to be 
\begin{align*}
\epsilon_{\text{final}} = \frac{2}{9} \left( 1 -  \frac{9}{10} \sqrt{\frac{\gamma}{\gamma + \kappa}} \right)^{R} \cdot  \left( \tilde f_t(\x_0) - \tilde f_t(\w_t^*) \right). 
\end{align*}

Let us initialize $\min_{\w}\, \bar f_{t,r} (\w)$ by $\x_0$. By definition, we have $\bar f_{t,r} (\w) \ge \tilde f_t (\w)$ and thus 
\begin{align*}
\epsilon_{\text{init}} & = \bar f_{t,r}(\x_0) - \bar f_{t,r} (\x_r^*) \\
& \le \tilde f_t(\x_0) - \tilde f_t (\x_r*) \\
& \le  \tilde f_t (\x_0) -  \tilde f_{t} (\w_t^*)  
\end{align*}
where we have used the fact that $\w_t^*$ is the minimizer of $\tilde f_{t} (\w)$ in the second inequality. 

This means we only need the initial versus final suboptimality of solving $\bar f_{t,r} (\w)$ to be 
\begin{align*}
\frac{1}{\eta} = \frac{\epsilon_{\text{init}}}{\epsilon_{\text{final}}} = \frac{9}{2} \left( 1 -  \frac{9}{10} \sqrt{\frac{\gamma}{\gamma + \kappa}} \right)^{-R},
\end{align*}
which, according to Lemma~\ref{lem:inexact-dane}, is achieved by inexact DANE with
  \begin{align*}
  K & = \calO \left(  \log \frac{1}{\eta} + \log \frac{\beta+\gamma+\kappa}{\gamma+\kappa} \right) \\
& = \calO \left( R \sqrt{\frac{\gamma}{\gamma + \kappa}} \right) \\
& = \calO \left( \log  n(\varepsilon) \right).
  \end{align*}
iterations. 

 Next we analyze the communication and computation efficiency of our algorithm.

  \textbf{Communication:} The total rounds of communication is
  \begin{align*}
  KRT & = \calO \left(  \log  n(\varepsilon) 
\cdot  \frac{b^{1/4} m^{1/2} \cdot \beta^{1/2} B^{1/2}}{ n(\varepsilon)^{1/4} \cdot L^{1/2}} \log n(\varepsilon)  \cdot \frac{n(\varepsilon)}{bm}  \right) \\
& =  \calO \left( \frac{ n(\varepsilon)^{3/4} \cdot \beta^{1/2} B^{1/2} }{ b^{3/4} m^{1/2} \cdot L^{1/2}}  \log^2 n(\varepsilon) \right) .
  \end{align*}

  \textbf{Computation:} Similar to the case of $b\le b^*$, for each DANE local subproblem \eqref{e:svrg-local_solver}, the sample size $b$ is larger than its condition number. Therefore, the total computational cost is
  \begin{align*}
  \calO (  bKRT ) & =  \calO \left( \frac{ b^{1/4} n(\varepsilon)^{3/4} \cdot \beta^{1/2} B^{1/2} }{  m^{1/2} \cdot L^{1/2}}  \log^2 n(\varepsilon) \right) .
  \end{align*}
  
\end{proof}

\section{Experiments}
\label{sec:expts}

\begin{table}[h]
\begin{center}
\caption{List of datasets used in the experiments.}
\label{tab:data}
\begin{tabular}{|c|c|c|c|}\hline
\hline
Name  & $\#$Samples & $\#$Features   &loss
\\
\hline\hline
codrna        & 271,617  & 8  & logistic\\
covtype       & 581,012 & 54 & logistic\\
kddcup99       &1,131,571   &127 & logistic \\
year       &463,715  &90  & squared \\
\hline
\end{tabular}
\end{center}
\end{table}

\begin{figure*}[t]
\psfrag{Objective in population}[][]{population objective}
\psfrag{Minibatch Size b}[][]{minibatch size $b$}
\psfrag{Minibatch SGD}[l][l][0.5]{Minibatch SGD}
\psfrag{Stochastic DANE(1,1)}[l][l][0.5]{MP-DANE, $K=1$}
\psfrag{Stochastic DANE(2,1)}[l][l][0.5]{MP-DANE, $K=2$}
\psfrag{Stochastic DANE(4,1)}[l][l][0.5]{MP-DANE, $K=4$}
\psfrag{Stochastic DANE(8,1)}[l][l][0.5]{MP-DANE, $K=8$}
\psfrag{Stochastic DANE(16,1)}[l][l][0.5]{MP-DANE, $K=16$}
\begin{tabular}{@{}c@{\hspace{0.01\linewidth}}c@{\hspace{0\linewidth}}c@{\hspace{0\linewidth}}c@{}}
& $m=4$ & $m=8$ & $m=16$ \\
\rotatebox{90}{\hspace*{4em}codrna} & 
\includegraphics[width=0.33 \textwidth]{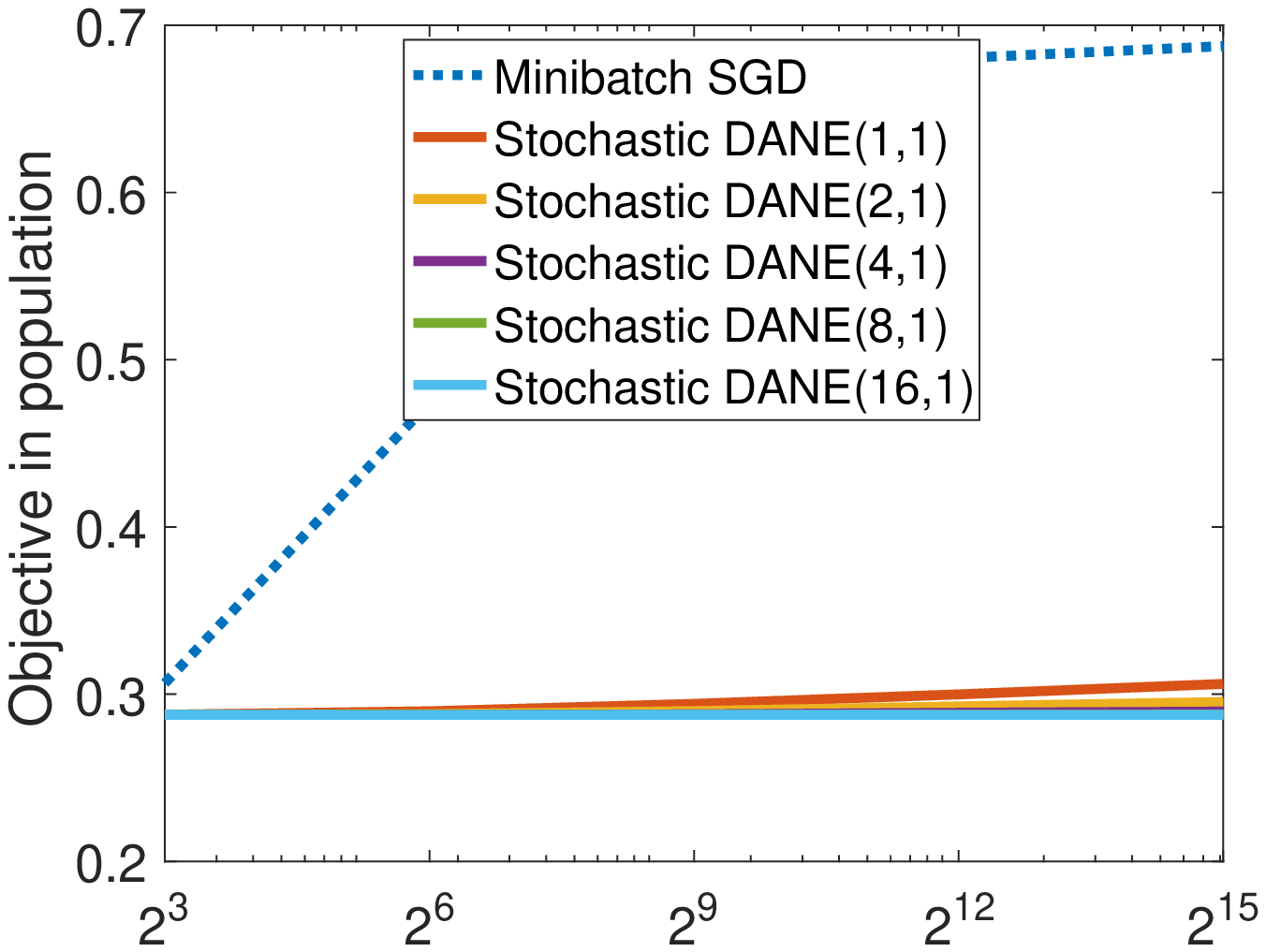} & 
\includegraphics[width=0.33 \textwidth]{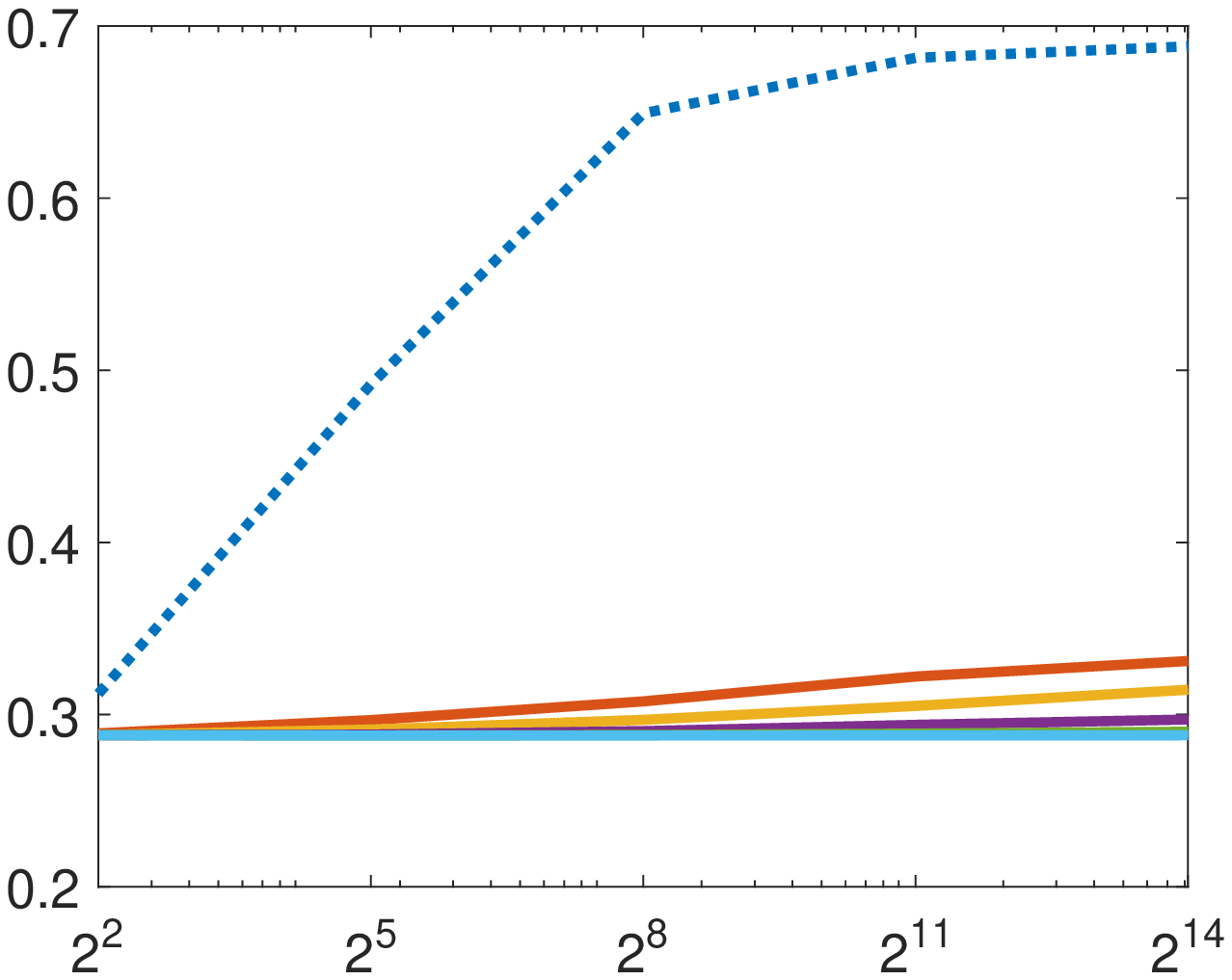} &
\includegraphics[width=0.33 \textwidth]{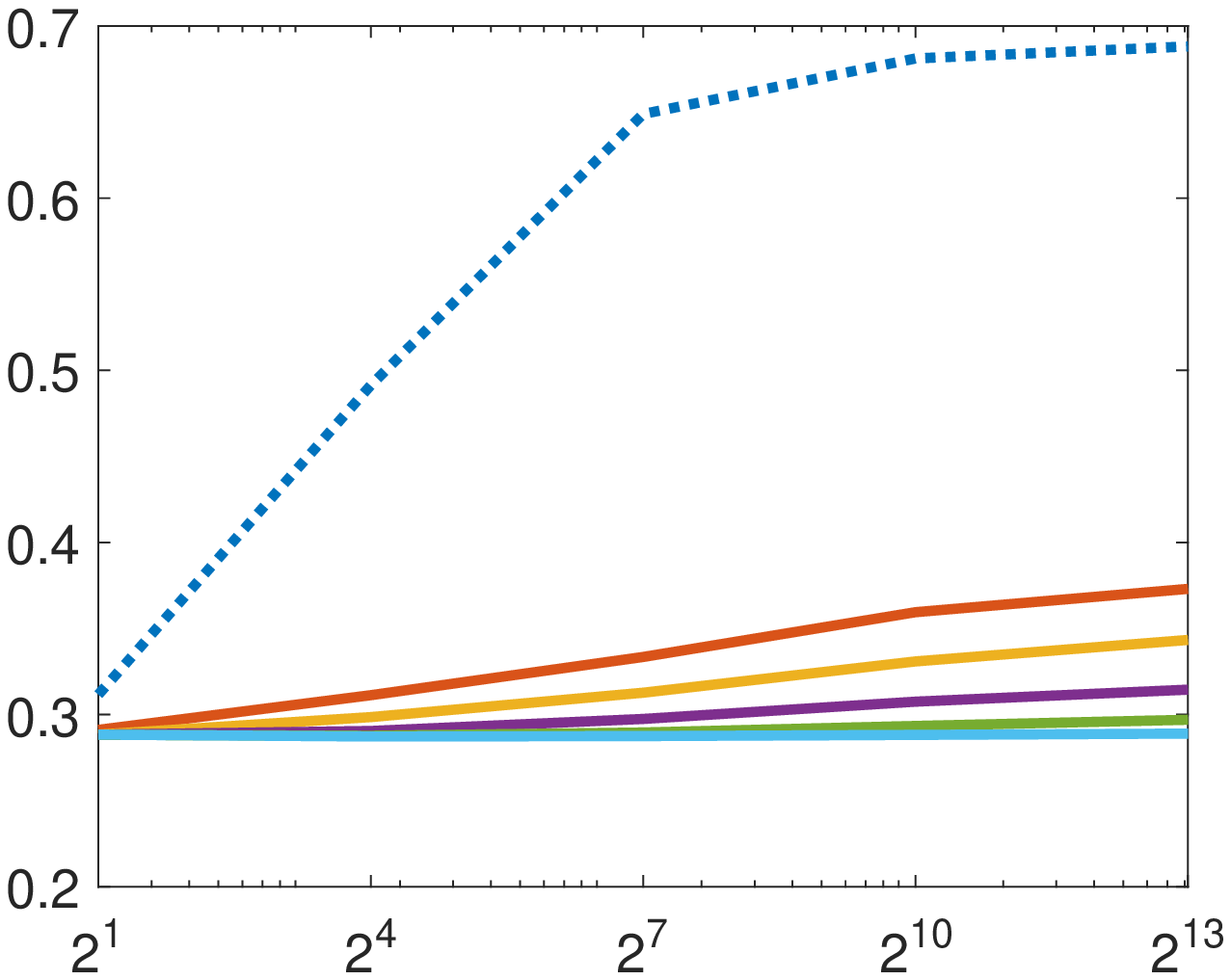} \\
\rotatebox{90}{\hspace*{3em}covtype} &
\includegraphics[width=0.33 \textwidth]{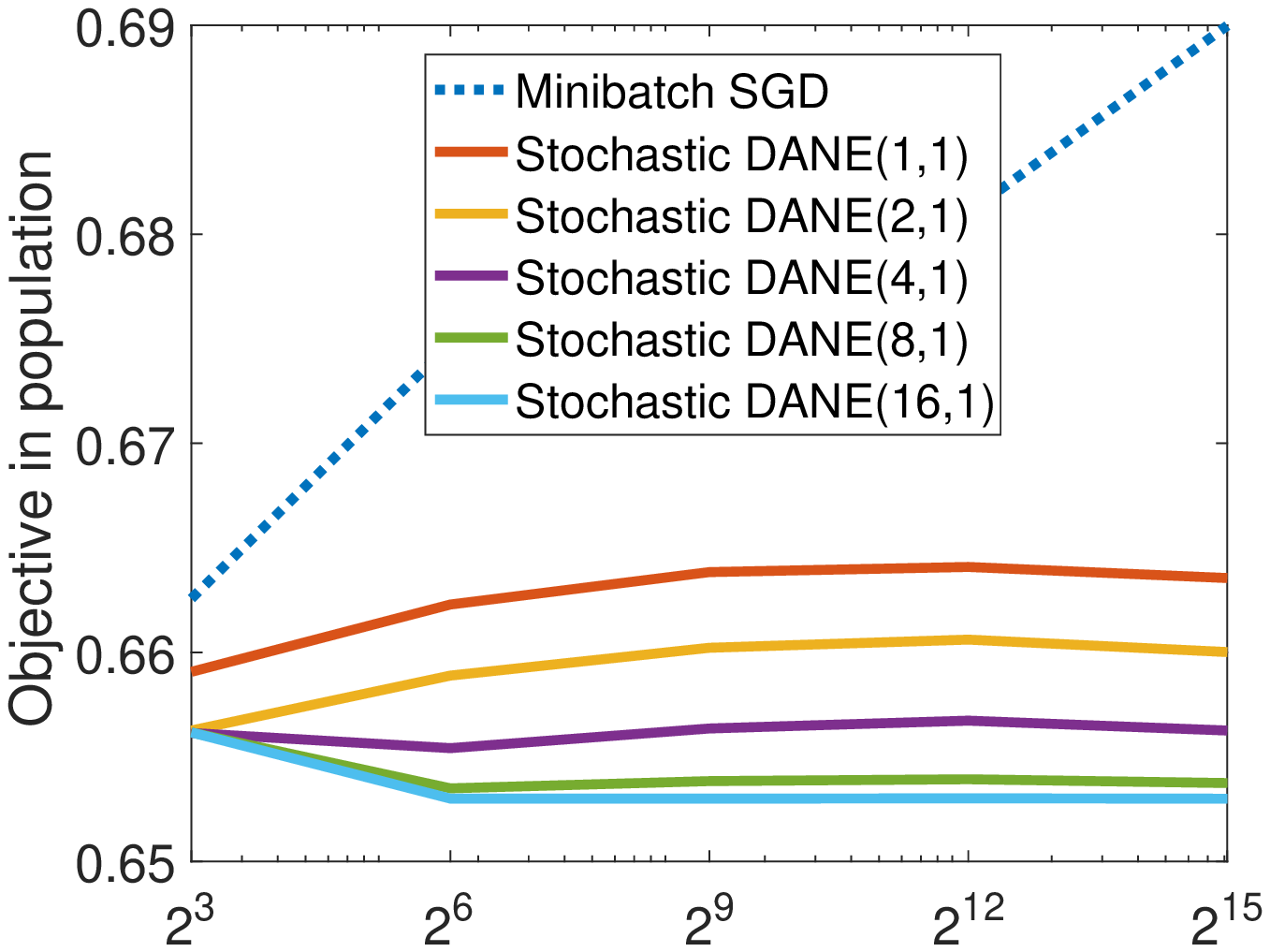} &
\includegraphics[width=0.33 \textwidth]{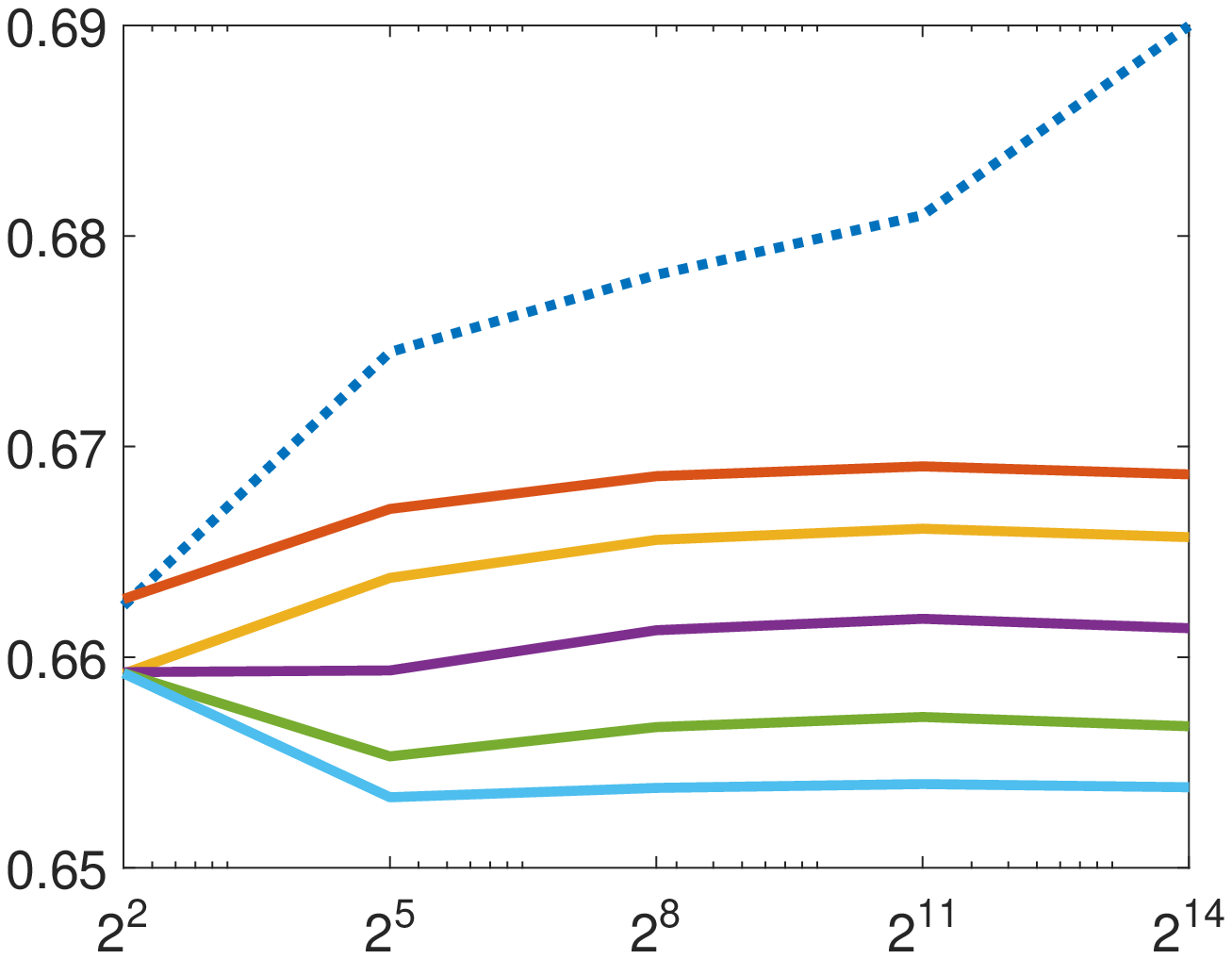} &
\includegraphics[width=0.33 \textwidth]{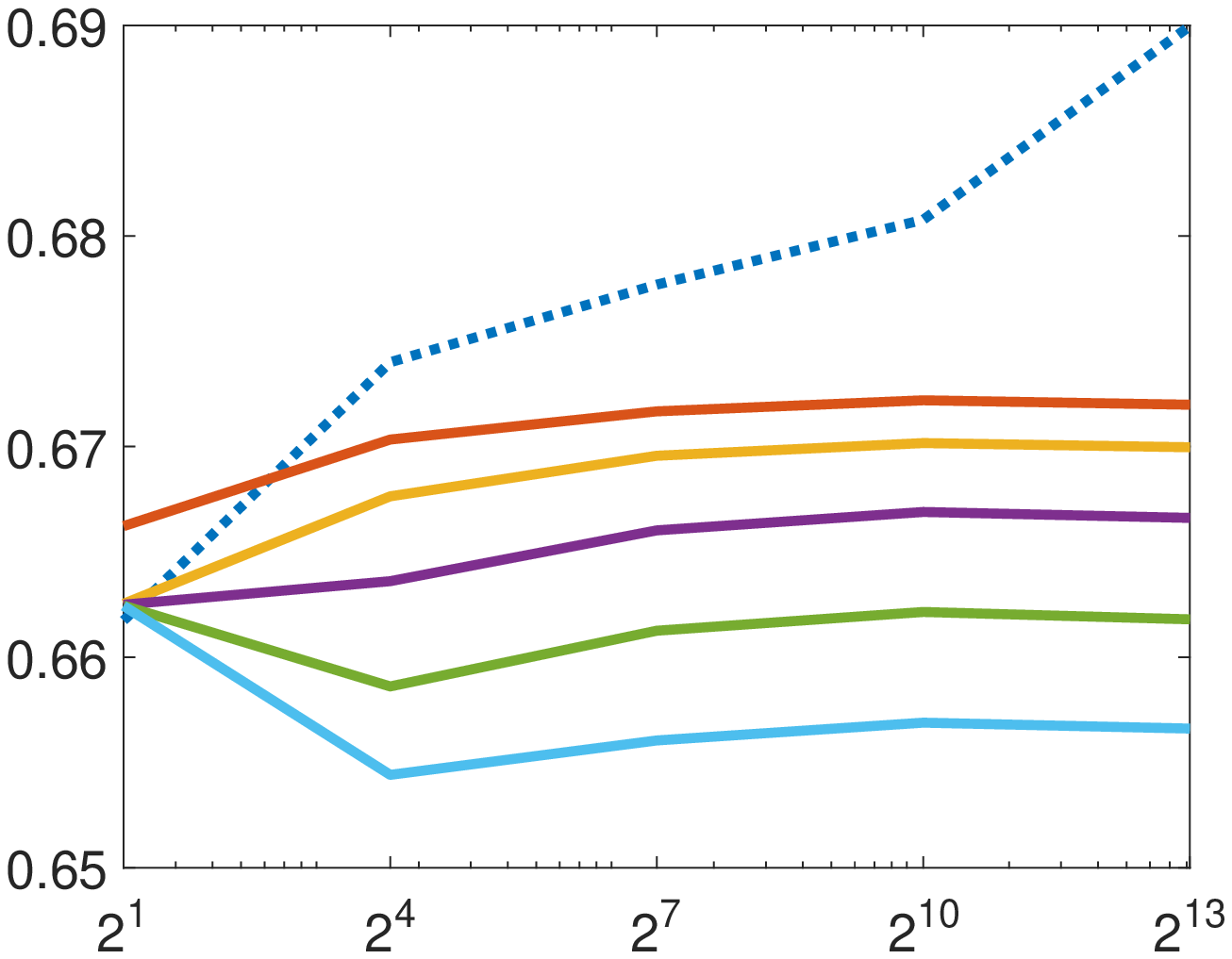} \\
\rotatebox{90}{\hspace*{3em}kddcup99} & 
\includegraphics[width=0.33 \textwidth]{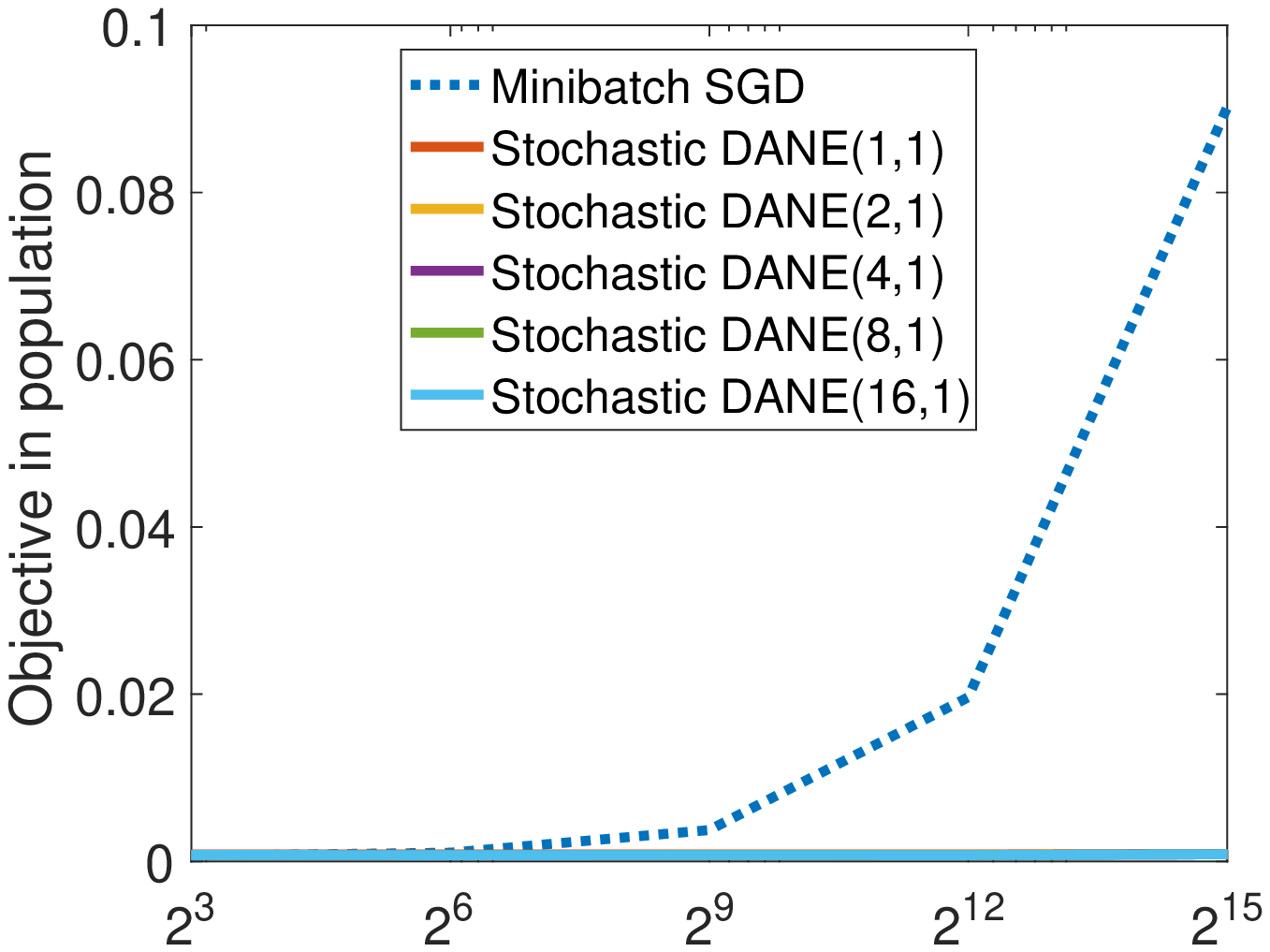} & 
\includegraphics[width=0.33 \textwidth]{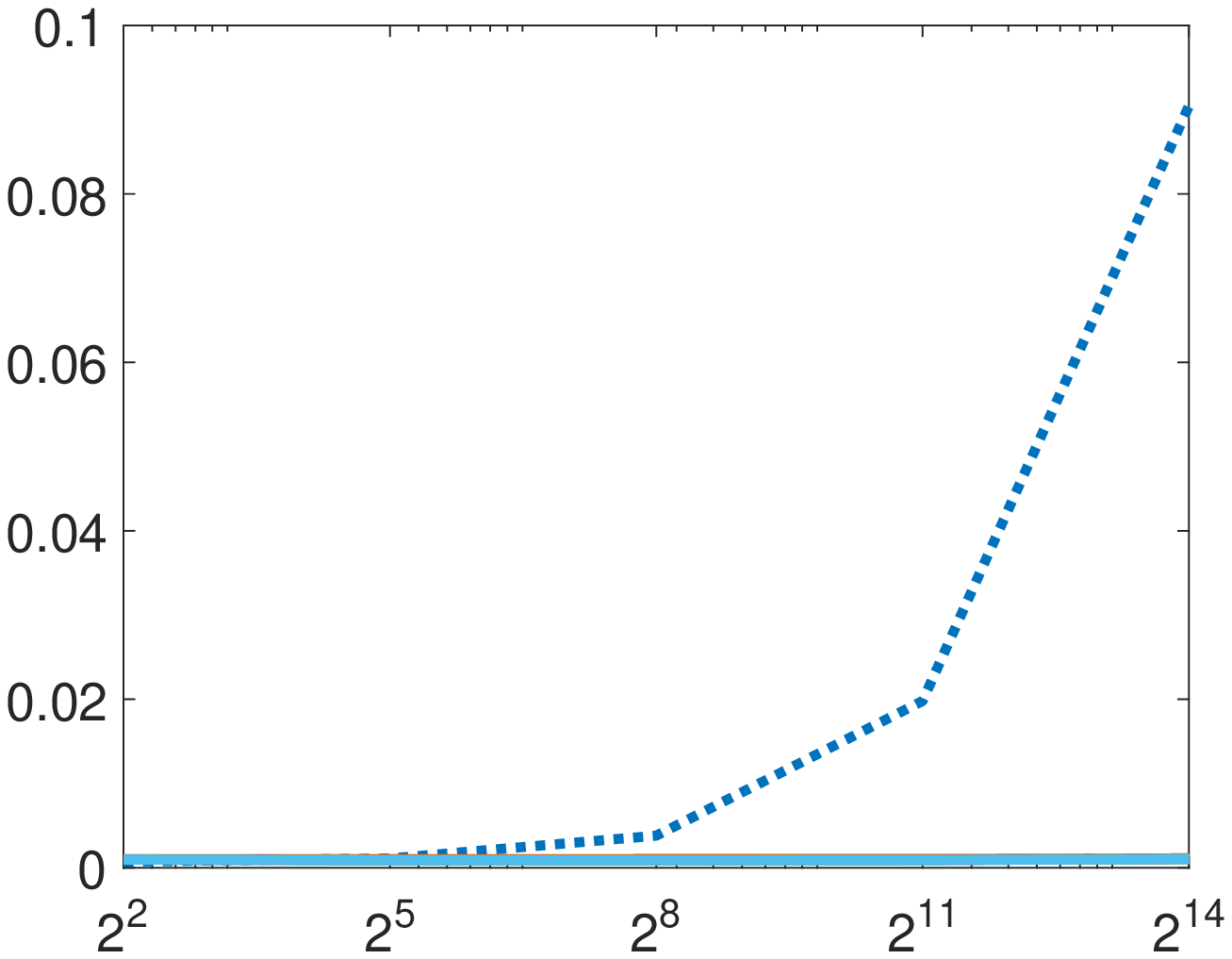} &
\includegraphics[width=0.33 \textwidth]{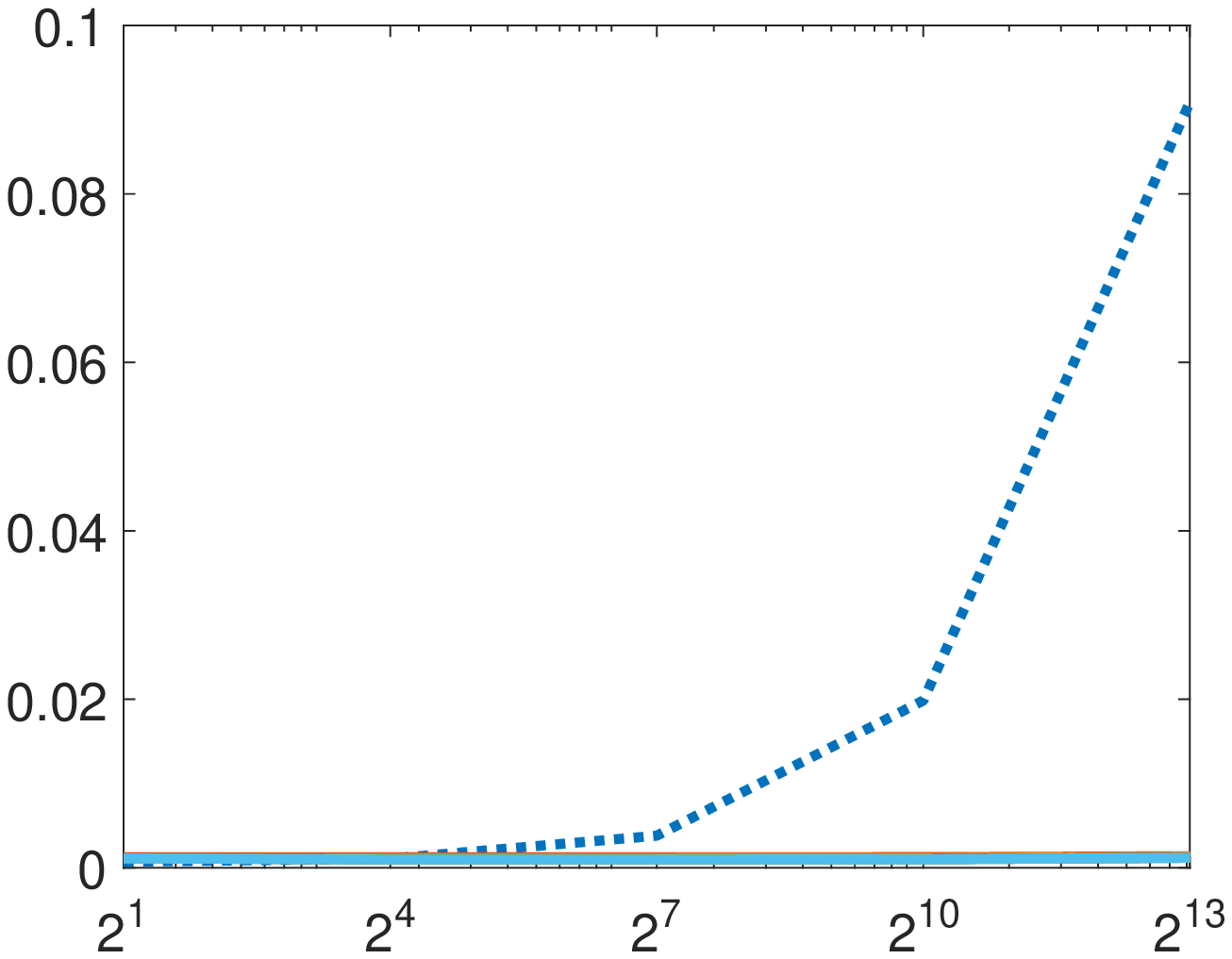} \\
\rotatebox{90}{\hspace*{4em}year} & 
\includegraphics[width=0.33 \textwidth]{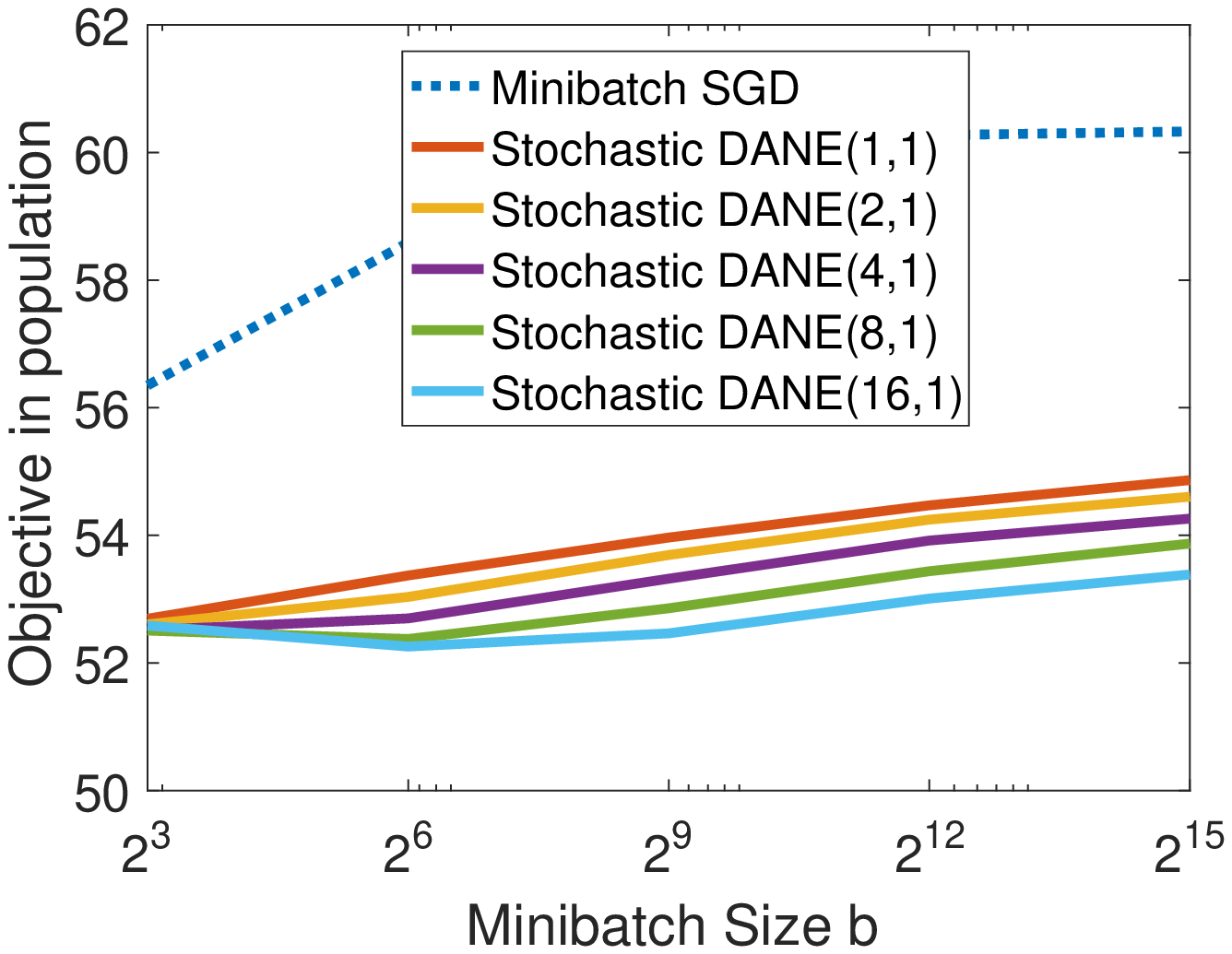} &
\includegraphics[width=0.33 \textwidth]{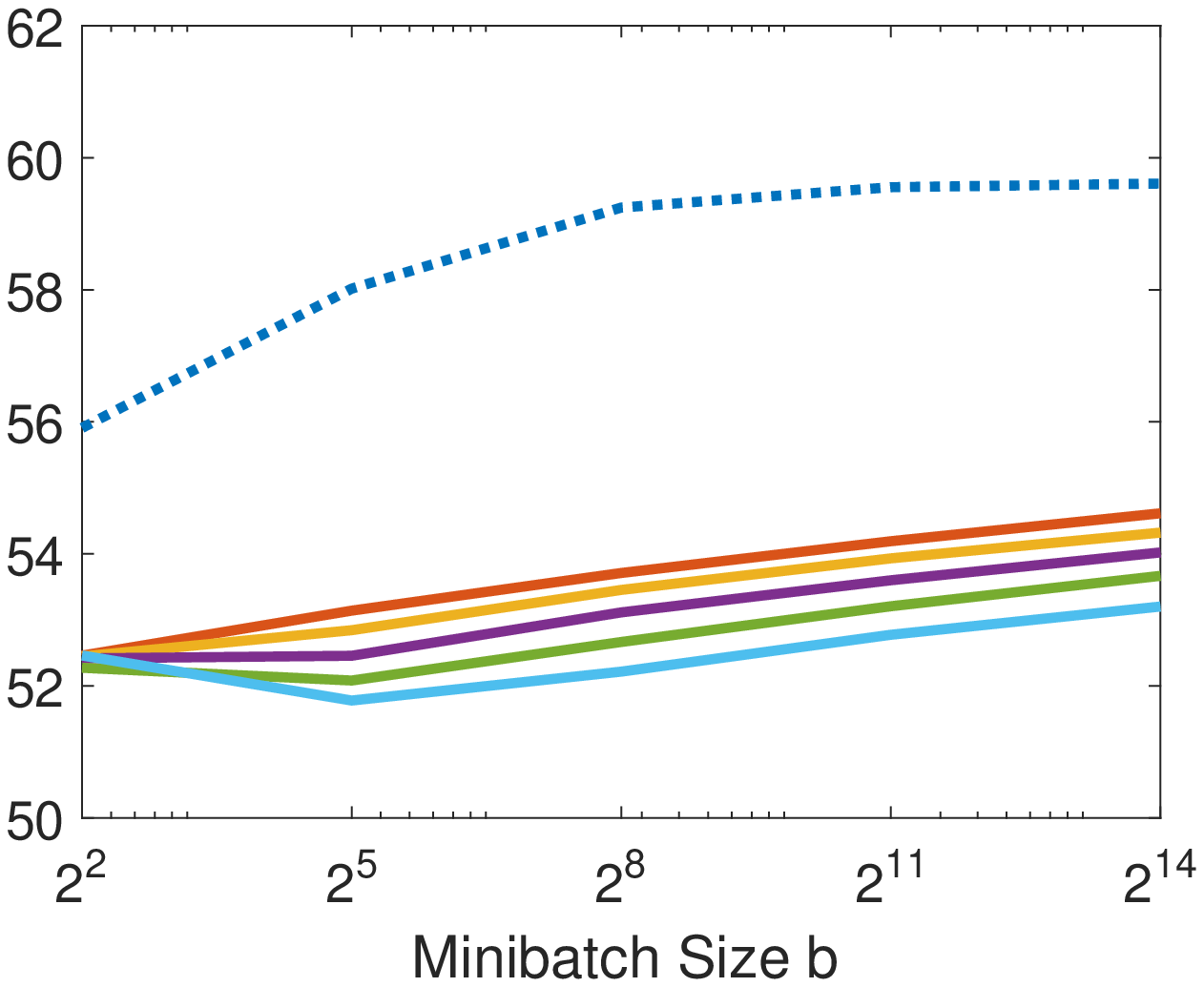} &
\includegraphics[width=0.33 \textwidth]{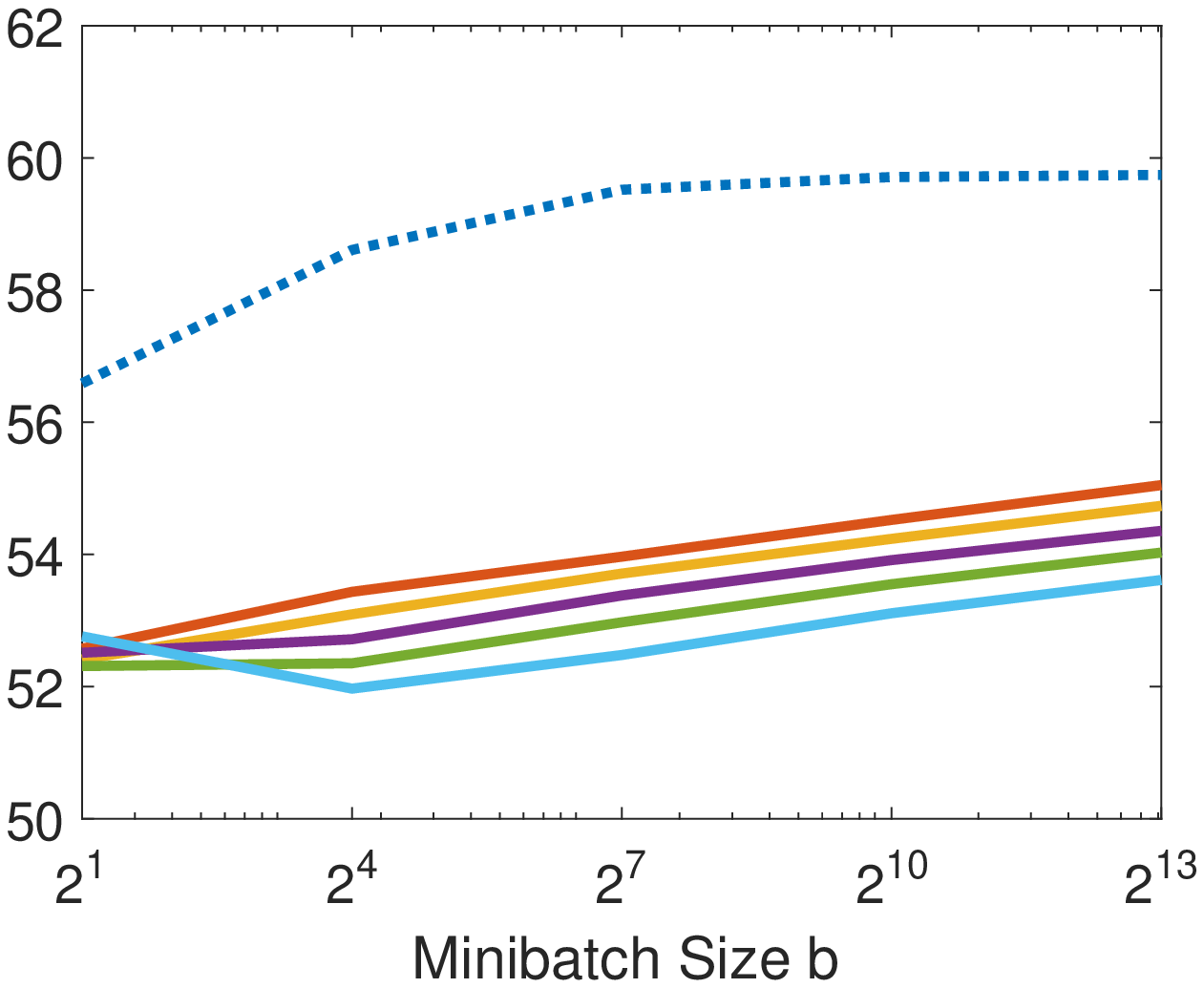} 
\end{tabular}
\caption{Illustration of the convergence properties of MP-DANE, for different minibatch size $b$, number of machines $m$, and number of DANE iterations $K$.}
\label{fig:exp}
\end{figure*}

In this section we present empirical results to support our theoretical analysis of MP-DANE. 
We perform least squares regression and classification on several publicly available datasets\footnote{\url{https://www.csie.ntu.edu.tw/~cjlin/libsvm/}}; the statistics of these datasets and the corresponding losses are summarized in Table~\ref{tab:data}. 
For each dataset, we randomly select half of the samples for training, and the remaining samples are used for estimating the stochastic objective. 

For MP-DANE, we use SAGA~\citep{defazio2014saga} to solve each local DANE subproblem~\eqref{e:svrg-local_solver} and fix the number of SAGA steps to $b$ (\ie, we just make one pass over the local data), while varying the number of DANE rounds $K$ over $\{1,2,4,8,16\}$. For simplicity, we do not use catalyst acceleration and set $R = 1$ and $\kappa = 0$ in all experiments. 
Our experiments simulate a distributed environment with $m$ machines, for $m = 4,8,16$. 
We conduct a simple comparison with minibatch SGD. Stepsizes for SAGA and minibatch SGD are set based on the smoothness parameter of the loss.

We plot in Figure~\ref{fig:exp} the estimated population objective vs. minibatch size $b$ for different parameters. 
We make the following observations.

\begin{itemize}
\item For minibatch SGD, as $b$ increases, the objective often increases quickly, this is because minibatch SGD can not uses large minibatch sizes while preserving sample efficiency.
\item For MP-DANE, the objective increases much more slowly as $b$ increases. This demonstrates the effectiveness of minibatch-prox for using large minibatch sizes. 
%% Also, in the experiments we just uses one passes of SAGA to approximately minimize the local DANE objective, and use only a few DANE iterations, thus still preserves the communication and computation efficiency compared with minibatch SGD. \weiran{There is no need to discuss runtime here, since they are not shown in the figure.}
\item  Running more iterations of DANE often helps, but with diminishing returns. This validates our theory that only a near-constant number of DANE iterations is needed for solving the large minibatch objective, without affecting the sample efficiency.
%% \item For large $m$, the objective typically slightly increases for MP-DANE. This is because when $m$ increases, each local subproblem will become less strongly convex since $\gamma_t$ decreases.  \weiran{not clear.}
\end{itemize}

\end{document}